\documentclass{article} % For LaTeX2e
\usepackage[preprint,nonanonymous]{neurips_2026}
\usepackage{hyperref}
\usepackage{url}

\usepackage{natbib}
\usepackage{url}
\usepackage[utf8]{inputenc} % allow utf-8 input
\usepackage[T1]{fontenc}    % use 8-bit T1 fonts
\usepackage{booktabs}       % professional-quality tables
\usepackage{amsfonts}       % blackboard math symbols
\usepackage{nicefrac}       % compact symbols for 1/2, etc.
\usepackage{microtype}      % microtypography
\usepackage{amsthm}
\usepackage{mathtools}
\usepackage{amssymb}
\usepackage{todonotes}

\usepackage[shortlabels]{enumitem}
\usepackage{xcolor}
\usepackage{comment}
\usepackage{subcaption}
\usepackage{graphicx}
\usepackage{caption}
\usepackage{makecell}
\usepackage{multirow}
\usepackage{float}
\usepackage{booktabs,caption}
\usepackage[flushleft]{threeparttable}
\usepackage{algorithmic}
\usepackage{algorithm}
\usepackage{xspace}

\usepackage{todonotes}

\usepackage{booktabs}
\usepackage{multirow}
\usepackage{graphicx}

\usepackage[table]{xcolor}
\usepackage{pgf}

\newcommand{\utilitycell}[3]{%
  \pgfmathparse{abs(#3-#2) < 0.00001 ? 1 : max(min((#1-#2)/(#3-#2),1),0)}%
  \let\ratio\pgfmathresult
  \pgfmathparse{0.96 - 0.18*\ratio} \let\red\pgfmathresult
  \pgfmathparse{0.78 + 0.20*\ratio} \let\green\pgfmathresult
  \pgfmathparse{0.78}               \let\blue\pgfmathresult
  \edef\temp{\noexpand\cellcolor[rgb]{\red,\green,\blue}}%
  \temp #1%
}

\newcommand{\asrcell}[3]{%
  \pgfmathparse{abs(#3-#2) < 0.00001 ? 1 : max(min((#3-#1)/(#3-#2),1),0)}%
  \let\ratio\pgfmathresult
  \pgfmathparse{0.96 - 0.18*\ratio} \let\red\pgfmathresult
  \pgfmathparse{0.78 + 0.20*\ratio} \let\green\pgfmathresult
  \pgfmathparse{0.78}               \let\blue\pgfmathresult
  \edef\temp{\noexpand\cellcolor[rgb]{\red,\green,\blue}}%
  \temp #1%
}

\newcommand{\sys}{\textsc{Warden}\xspace}

\newtheorem{theorem}{Theorem}

\numberwithin{theorem}{section}

\title{Information Theoretic Adversarial Training of Large Language Models}

\author{
Yiwei Zhang \\
Purdue University \\
West Lafayette, IN 47907, USA\\
\texttt{yiweizhang@purdue.edu} 
\And
Jeremiah Birrell \\
Texas State University \\
San Marcos, TX 78666, USA \\
\texttt{jbirrell@txstate.edu} \\
\And
Reza Ebrahimi \\
University of South Florida \\
Tampa, FL  33620, USA\\
\texttt{ebrahimim@usf.edu} 
\And
Rouzbeh Behnia\\
University of South Florida \\
Tampa, FL  33620, USA\\
\texttt{behnia@usf.edu} 
\And
Jason Pacheco \\
University of Arizona \\
Tucson, AZ, USA\\
\texttt{pachecoj@cs.arizona.edu} 
\And
Elisa Bertino \\
Purdue University \\
West Lafayette, IN 47907, USA\\
\texttt{bertino@purdue.edu} 
}

\begin{document}
\maketitle

\begin{abstract}
Large language models (LLMs) remain vulnerable to adversarial prompting despite advances in alignment and safety, often exhibiting harmful behaviors under novel attack strategies. While adversarial training can improve robustness, existing approaches are computationally expensive and difficult to scale. Recent continuous adversarial training methods, such as Continuous adversarial training (CAT) and Continuous Adversarial Preference Optimization (CAPO), address this challenge by leveraging gradient-based perturbations in the embedding space, enabling more efficient and expressive attacks. Building on this paradigm, we propose \sys, a distributionally robust adversarial training framework for LLMs that dynamically reweights adversarial examples through an $f$-divergence ambiguity set around the empirical training distribution. Our method optimizes the worst-case adversarial loss within a divergence ball around the empirical data distribution, automatically emphasizing harder adversarial examples. Using the convex dual formulation, the objective reduces to a log-sum-exp form under the KL divergence, with a dynamical parameter controlling the strength of reweighting. This study leads to a new class of information-theoretic objectives that significantly reduce attack success rates while maintaining model utility. Across multiple LLMs and attack settings, \sys substantially reduces attack success rates with computational and utility costs comparable to CAT-, CAPO-, and MixAT-based baselines, making it a practical approach for scalable robust alignment.
% Empirically, our method achieves robustness improvements with computational cost comparable to CAT and CAPO, making it a practical approach for scalable adversarial training of LLMs.
\end{abstract}

\section{Introduction}\label{sec:intro}
Despite recent advances in alignment and safety, large language models (LLMs) remain vulnerable to adversarial prompting and can exhibit harmful or unintended behaviors under novel attack strategies. Recent work highlights that standard training and alignment procedures often fail to anticipate such failure modes, leaving models exposed to unforeseen adversarial inputs \citep{anil2024many_redteaming}. Moreover, empirical evidence shows that harmful behaviors can persist even after safety fine-tuning, and that adversarial training is necessary to meaningfully improve robustness \citep{sheshadri2025latent_latent}.

While effective, adversarial training of LLMs tends to be computationally expensive and prohibitive. Continuous adversarial training methods, such as Continuous Adversarial Training (CAT), Continuous Adversarial Preference Optimization (CAPO), have revealed a promising direction for improving the robustness of LLMs by replacing discrete token-level perturbations with continuous, gradient-based adversarial attacks that are efficient for LLMs \citep{xhonneux2024efficient_CAT}. 
This continuous relaxation enables 
% By operating in the embedding space, CAT and CAPO enable 
scalable adversarial training that better aligns with the optimization dynamics of modern LLMs. 
Subsequent work, such as MixAT, further improves this paradigm by combining continuous and discrete adversarial training~\citep{dekany2025mixat}. 
Nevertheless, existing adversarial training objectives still typically aggregate per-sample adversarial losses uniformly, which can underemphasize rare but high-loss adversarial examples.
% However, despite these advances, there is still room for improvement in continuous adversarial training. 

%%Recent work has shown that the behavior of CAT can be better understood through the lens of in-context learning, revealing both its strengths and its inherent constraints in capturing complex adversarial distributions and generalizing across tasks \cite{CAT_ICLR_26}. 

Building on the work of \citep{xhonneux2024efficient_CAT}, we introduce \sys{} (\textbf{W}orst-case \textbf{A}dversarial \textbf{R}eweighting via $f$-\textbf{D}iverg\textbf{EN}ces), a distributionally robust optimization (DRO) framework for adversarial training of LLMs based on $f$-divergence reweighting (as illustrated in Figure~\ref{fig:method_workflow}).
% a Distributionally Robust Optimization (DRO)–based $f$-divergence reweighting method for adversarial training of LLMs. 
% The core idea is to optimize the worst-case expected adversarial loss within a divergence ball around the empirical training distribution, which automatically assigns larger weights to harder adversarial samples. 
Rather than minimizing the empirical average adversarial loss, \sys minimizes the worst-case expected adversarial loss over distributions within a divergence ball around the empirical training distribution. 
Using the convex dual formulation, the objective becomes a log-sum-exp form under KL divergence, and the dual variable $\lambda$ controls the strength of reweighting. We explore two strategies for $\lambda$: 
% Fixed: $\lambda$ is set to a constant value. 
Learnable: $\lambda$ is treated as a trainable parameter. Optimized: $\lambda$ is computed at each iteration using the bisection method derived from the DRO dual. Our work makes several major contributions to adversarial training of LLMs:  
% {\bf \color{blue} JB: I don't think we want to talk about the fixed method outside of the ablation section, because it doesn't work well and it makes the reweighting mechansim  seem more trivial.}

\begin{itemize}
    % \item We propose a new class of information-theoretic based objectives that significantly reduce attack success rate while increasing utility of LLMs.
    \item We propose \sys, a modular DRO-based reweighting framework for continuous adversarial training of LLMs, yielding a class of information-theoretic objectives that emphasize high-loss adversarial examples.
    \item We derive a tractable KL-DRO objective with a log-sum-exp form and study practical strategies for controlling the reweighting strength through the dynamical dual variable $\lambda$.
    \item Empirically, \sys reduces attack success rates across multiple LLMs and attack settings while maintaining utility and computational costs comparable to strong continuous adversarial training baselines.
    % \item The empirical computational complexity of our proposed class of objectives is on par with that of CAT and CAPO proposed in \cite{xhonneux2024efficient_CAT}.
\end{itemize}

% The source code is included in the Supplementary Material.

\begin{figure*}[t]
    \centering
    \includegraphics[width=\textwidth]{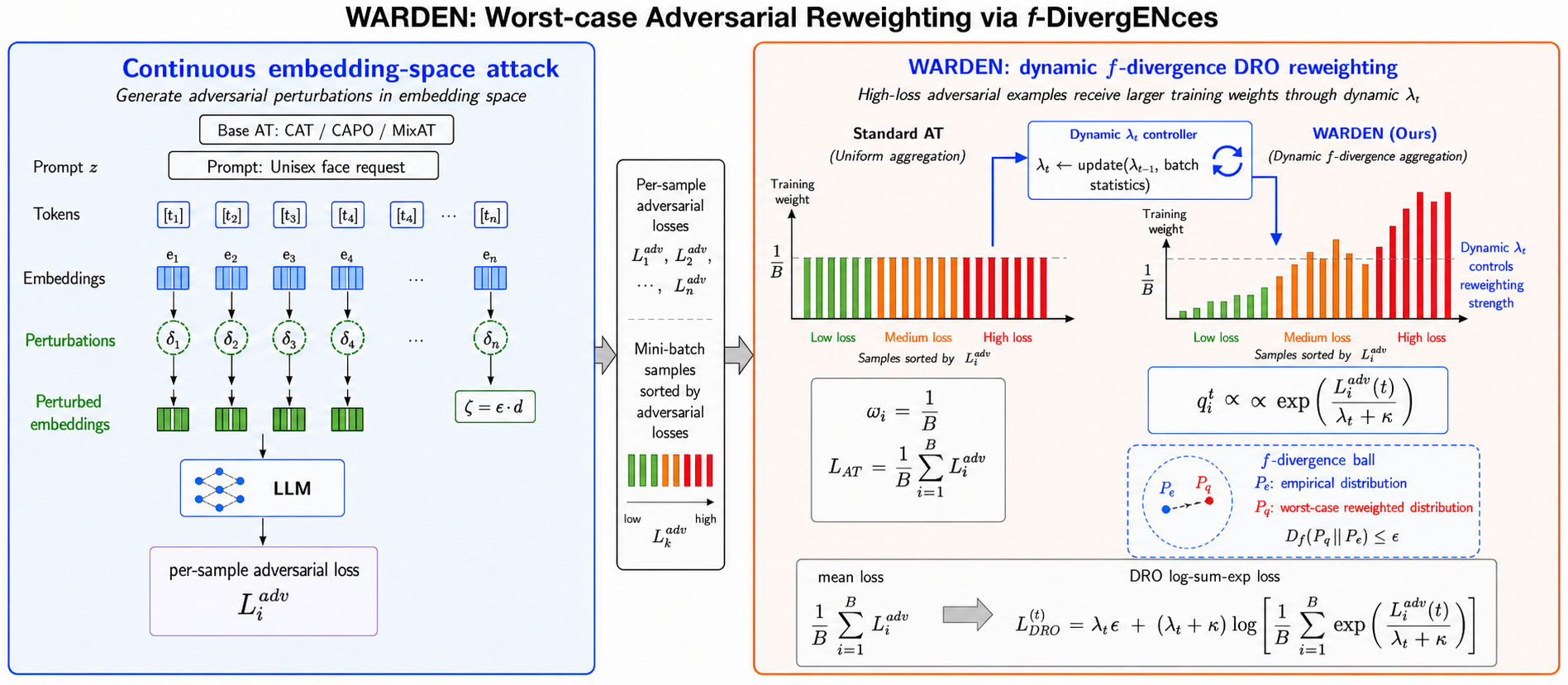}
    \caption{
    Overview of \sys{} with adaptive DRO reweighting.
    \textbf{Left:} A base continuous adversarial training method generates embedding-space perturbations and per-sample adversarial losses.
    \textbf{Right:} \sys{} replaces uniform aggregation with an $f$-divergence DRO objective that dynamically upweights high-loss adversarial examples via a learnable or optimized dual variable $\lambda_t$.
    For KL divergence, the dual objective reduces to a log-sum-exp form while leaving the underlying continuous attack pipeline unchanged.
    }
    \label{fig:method_workflow}
\end{figure*}

\section{Related Work}\label{sec:related}

%%\todo[inline]{Needs further development and streamlining}

Closely relevant work to this study can be categorized into three major streams.

\textit{Continuous Adversarial Training in the LLM's embedding space.} Adversarial training has been known as a leading approach for improving robustness of AI models; however, for LLMs, it is often impractical due to the high computational cost of generating discrete attacks during training. \cite{xhonneux2024efficient_CAT} addressed this limitation by proposing CAPO, an adversarial variant of Identity Preference Optimization (IPO) \citep{azar2024general_IPO} and the CAT algorithm that perform adversarial attacks in the LLM's continuous embedding space, making the process significantly more efficient. While CAT requires a utility dataset to prevent collapsing to degenerate states that refuse safe prompts, CAPO does not rely on any utility dataset. 
Highlighting the effectiveness of continuous adversarial training of LLMs,  \citep{fu2025shortlength_ICDL2025_Orthogonal} reveal theoretical and practical evidence that adversarial training with adversarial suffixes of a certain length yields defending against attacks with quadratically larger lengths. In a subsequent work, \citet{dekany2025mixat} propose MixAT to further improve the performance of continuous adversarial training by incorporating some discrete adversarial attacks in the token space while maintaining a runtime that is comparable to CAPO and CAT.
More recently, \citet{CAT_ICLR_26} offers a theoretical analysis from the in-context learning lens that sheds light on why adversarial perturbations in the embedding space help LLMs defend against jailbreak prompts in the original input token space. They also show that the behavior of continuous adversarial training can be improved by adding a regularizer inspired by in-context learning.
Our work is closest to this line, but differs in focus: rather than modifying the attack generation process, \sys changes how per-sample adversarial losses are aggregated by replacing uniform averaging with distributionally robust reweighting.

\textit{Adversarial training objectives and reweighting.}
A broader body of work studies the generalization, stability, and robustness--utility trade-offs of adversarial training \citep{wang2024benign,xiao2024uniformly,li2025adversarial,altinisik2025explaining,xie2024high,zhang2024stability}, while other methods improve adversarial training through objective or optimization design, such as perturbation reweighting, self-distillation, robust distillation, critical-layer fine-tuning, and multi-norm robustness objectives \citep{zhang2024improving,cho2025long,lee2025indirect,gopal2025boosting,jiang2024ramp}. Most related to \sys are methods that treat training signals non-uniformly, including pixel-level perturbation reweighting, long-tailed adversarial training, and vulnerable-data-aware adversarial training \citep{zhang2024improving,cho2025long,feng2025vulnerable}. \sys shares the intuition that difficult adversarial examples should receive greater emphasis, but realizes this through an $f$-divergence DRO objective for LLM adversarial training, yielding a lightweight loss-aggregation layer that leaves the continuous attack pipeline unchanged.

\textit{Orthogonal LLM defenses.}
Other LLM defense mechanisms are complementary to our approach. Latent adversarial training perturbs internal model states to mitigate harmful behaviors without relying on explicit failure-inducing prompts \citep{sheshadri2025latent_latent}. Shadow-LLM defenses instead use an auxiliary LLM to protect a target model at inference time \citep{wang2025selfdefend_Usenix_Orthogonal}. In contrast, \sys is a training-time method that directly improves the robustness of the target LLM and can potentially be combined with these orthogonal defenses.

% While this stream is the closest to this study, for completeness, we provide a non-exhaustive list of two other orthogonal but related streams.

% \textit{Latent Adversarial Training.} Another orthogonal stream of work is latent adversarial training which focuses on curbing against failure modes without examples that elicit them \citep{sheshadri2025latent_latent}. To this end, the adversarial perturbations are applied to the model’s latent state instead of its inputs. This approach decouples the adversarial training from the input samples and can be used in conjunction with any continuous adversarial training in the embedding space, which is out of the scope of our inquiry in this study.

% \textit{Shadow LLMs.} There has been another orthogonal recent stream of work that use a shadow LLM as a defense instance to concurrently protect a target LLM instance \citep{wang2025selfdefend_Usenix_Orthogonal}.

% \section{Proposed Method: Information Theoretic Adversarial Training of Large Language Models}
\section{\sys: Distributionally Robust Adversarial Training of LLMs with Worst-Case Adversarial Reweighting}\label{sec:sys}

Continuous adversarial training methods improve LLM robustness by generating adversarial perturbations in the embedding space and optimizing the model against the resulting adversarial losses. However, existing objectives typically aggregate per-sample adversarial losses uniformly. This average-case aggregation can underemphasize rare but high-loss adversarial examples, even though such examples often correspond to the most informative failure modes for robust alignment. We address this limitation by replacing uniform aggregation with a distributionally robust reweighting objective that emphasizes hard adversarial samples while preserving the underlying continuous attack pipeline.

\subsection{Worst-Case Adversarial Reweighting}

Continuous adversarial training for LLMs generates embedding-space perturbations and trains the model to prefer desired responses over undesired ones under these perturbations. Given a distribution $(x,y,\hat{y})\sim P_n$  of harmful prompts $x$, desired responses $y$, and undesired responses $\hat{y}$, and (continuous-attack)   adversarial perturbation $\delta(x,\hat{y})$ targeted at $\hat{y}$, the standard preference-based objective~\citep{xhonneux2024efficient_CAT} minimizes
% To maximize the likelihood of a desired response while decreasing the likelihood of an undesired response, the continuous adversarial training approach of  \cite{xhonneux2024efficient_CAT} is given by
\begin{align}\label{eq:Xhonneux_original}
    &\inf_\theta E_{(x,y,\hat{y})\sim P_n}\left[\mathcal{L}_\theta(x,y,\hat{y})\right]\,,
\end{align}
where
\begin{align}
    &\mathcal{L}_\theta(x,y,\hat{y})\coloneqq -\ell_\beta\left(\log\frac{f_\theta(y|x+\delta(x,\hat{y}))}{f_{\theta_0}(y|x)}-\log\frac{f_\theta(\hat{y}|x+\delta(x,\hat{y}))}{f_{\theta_0}(\hat{y}|x)}\right)\,,\,\,
    \ell_\beta(h)\coloneqq\left(h-\frac{1}{2\beta}\right)^2\notag
\end{align}
and 
$f_{\theta_0}$ 
denotes the original model.
% denotes the reference model, $P_n$ is the empirical distribution over triples $(x,y,\hat{y})$, and $\delta(x,\hat{y})$ denotes the adversarial perturbation of $x$ targeted at the unsafe response $\hat{y}$, obtained by approximately solving the inner continuous attack problem.
To prevent the LLM from collapsing to degenerate behaviors and refusing safe prompts, this loss implicitly minimizes the Kullback-Leibler divergence with regards to the original model distribution $f_{\theta_0} (y|x)$.

Instead of minimizing the empirical average in Eq.~\eqref{eq:Xhonneux_original}, \sys{} minimizes a worst-case expected loss over distributions close to the empirical distribution by introducing
% In this work, we augment the method \eqref{eq:Xhonneux_original} utilizing 
a computationally inexpensive DRO-reweighting layer, which has the effect of focusing the training on the more difficult adversarial samples.  Specifically, let $P_n=\frac{1}{n}\sum_{i=1}^n \delta_{(x_i,y_i,\hat{y}_i)}$ denote the empirical distribution training set.
For an $f$-divergence $D_f$, we define the distributionally robust adversarial objective
% we propose training using an DRO-based variant of \eqref{eq:Xhonneux_original} which utilizes $f$-divergence reweighting.  Specifically, we propose the method
\begin{align}\label{eq:reformulation}
&\inf_{\theta\in\Theta}\sup_{Q: D_f(Q\|P_n)\leq \epsilon}\{E_{(x,y,\hat{y})\sim Q}[\mathcal{L}_\theta(x,y,\hat{y})]-\kappa D_f(Q\|P_n)\}.
\end{align}
The radius $\epsilon$ limits how far the reweighting distribution $Q$ may move from $P_n$, while $\kappa>0$ adds a soft-constraint divergence penalty which acts as a   regularization term for numerical stability. The inner maximization therefore emphasizes high-loss adversarial examples on which the current model is most vulnerable.

Using convex duality, Eq.~\eqref{eq:reformulation} can be written as
\begin{align}
&\inf_{\theta\in\Theta}\inf_{\substack{\lambda\geq0,\\\rho\in\mathbb{R}} }\!\left\{\!\lambda \epsilon+\rho+(\lambda+\kappa) E_{(x,y,\hat{y})\sim P_n}\!\!\left[f^*\left(\frac{\mathcal{L}_\theta(x,y,\hat{y})-\rho}{\lambda+\kappa}\right)\right]\right\}\notag\,,
\end{align}
where $f^*$ denotes the Legendre transform of $f$. The derivation is provided in Appendix~\ref{app:DRO_duality_proof}. This dual form shows that \sys{} can be implemented as a lightweight loss-aggregation layer on top of existing continuous adversarial training methods, involving two additional dynamical parameters, $\lambda$ and $\rho$.
% where the equality follows from convex duality; see Appendix \ref{app:DRO_duality_proof} for  detailed derivation. 

\subsection{KL-DRO Objective and Adaptive Reweighting}
In this work, we focus on the KL-divergence instantiation of  Eq.~\ref{eq:reformulation}. For $D_f(Q\|P_n)=\mathrm{KL}(Q\|P_n)$, the dual objective simplifies to 
\begin{align}\label{eq:reformulation_final_KL}
&\inf_{\theta\in\Theta}\sup_{Q: \mathrm{KL}(Q\|P_n)\leq \epsilon}\{E_{(x,y,\hat{y})\sim Q}[\mathcal{L}_\theta(x,y,\hat{y})]-\kappa \mathrm{KL}(Q\|P_n)\}\\
=&\inf_{\theta\in\Theta}\inf_{\lambda\geq0 }\!\left\{\!\lambda \epsilon+(\lambda+\kappa)\log E_{(x,y,\hat{y})\sim P_n}\!\left[\exp\left(\frac{\mathcal{L}_\theta(x,y,\hat{y})}{\lambda+\kappa}\right)\right]\right\}\notag\,.
\end{align}
This log-sum-exp objective interpolates between average-case and worst-case training: larger $\lambda+\kappa$ approaches uniform averaging, while smaller values concentrate weight on high-loss samples. Hence, $\lambda$ controls the reweighting strength.

In practice, \sys uses adaptive treatments for setting $\lambda$. We consider two main variants. 
In the learnable variant, $\lambda$ is optimized jointly with the model parameters. In the optimized variant, $\lambda$ is recomputed at each training iteration by solving the one-dimensional convex dual problem in Eq.~\eqref{eq:reformulation_final_KL}. Since the objective is convex in $\lambda$, this update can be performed efficiently with a bisection solver; see Appendix~\ref{app:bisection_lambda}. 
In both adaptive variants, the resulting $\lambda_t$ changes across training iterations and minibatches, allowing \sys{} to adjust the amount of reweighting to the current adversarial loss distribution.
For completeness, we also evaluate a fixed-$\lambda$ setting in the ablation study, where $\lambda$ is held constant as a hyperparameter. We do not treat this setting as a primary variant of \sys{}, since it lacks minibatch-adaptive control of the DRO reweighting strength and performs substantially worse in our experiments; see Section~\ref{subsec:ablation}.

% {\bf \color{blue} JB: I don't think we should promote $\lambda$-fixed as a version of our method; it really isn't and since it doesn't work well then we should make a clear distinction between our approach and that one.  If we do mention $\lambda$-fixed here, it should be after the other two  approaches   and we should immediately say that $\lambda$ fixed doesn't work well; this should help make our implementation  seem less trivial. }

% Note that, in addition to a hard DRO constraint, $D_f(Q\|P_n)\leq \epsilon$, \eqref{eq:reformulation_final} employs a soft-constraint penalty term $\kappa D_f(Q\|P_n)$, $\kappa>0$, as a regularization to stabilize the numerical implementation; see Appendix \ref{app:bisection_lambda} for further discussion.

% \subsection{Including Utility Loss}
\subsection{DRO Reweighting with Utility Preservation}
Some adversarial training methods include a utility dataset, $\mathcal{D}_u$ to preserve benign instruction-following ability and prevent excessive refusal.
In this setting, we similary apply DRO reweighting only to the adversarial loss and add the standard utility loss separately.  For the continuous attack of \cite{xhonneux2024efficient_CAT},  we have the following KL-divergence variant with utility regularization
\begin{align}\label{eq:with_utility}
&\inf_{\theta\in\Theta}\left\{ \inf_{\lambda\geq0}\left\{\lambda\epsilon+(\lambda+\kappa) \log E_{(x,y,\hat{y})\sim P_n}\!\!\left[\exp\!\left(\frac{\mathcal{L}_\theta(x,y,\hat{y})}{\lambda+\kappa}\right)\right]  \right\} -E_{(x_u,y_u)\sim \mathcal{D}_u}[\log f_\theta(y_u|x_u)]     \right\}\,,\\
&\mathcal{L}_\theta(x,y,\hat{y})\coloneqq \log f_\theta(\hat{y}|x+\delta(x,\hat{y}))-\log f_\theta(y|x+\delta(x,\hat{y}))\,,\notag
\end{align}
which should be compared with  Eq.~(4) of \cite{xhonneux2024efficient_CAT}.
This objective preserves the base adversarial loss while replacing its uniform empirical average with KL-DRO aggregation.

The same DRO-reweighting mechanism can be added on top of stronger base adversarial training pipelines - MixAT method; see Eq.~(7) in \cite{dekany2025mixat}. In the KL case this  again has the form \eqref{eq:with_utility}, only with the loss function changed to
\begin{align}
   \mathcal{L}_\theta(x,y,\hat{y})\coloneqq \log f_\theta(\hat{y}|\hat{x})-\log f_\theta(y|\hat{x})\,,
\end{align}
where $\hat{x}$ is the adversarial sample generated by the MixAT attack:
\begin{align}
    \hat{x}=\text{argmax}_{x^\prime\in \mathcal{N}(x)}\log f_\theta(\hat{y}|x^\prime)\,.
\end{align}
Thus, \sys{} is modular: it does not modify the adversarial example generation procedure, optimizer, perturbation budget, or utility objective of the base method. It only replaces the aggregation of per-sample adversarial losses with a distributionally robust objective. 
For more general $f$-divergences, similar methods can be developed by generalizing \eqref{eq:reformulation}.  However, in this work we focus on the KL variants.

% \section{Experiments}
\section{Experimental Details}\label{sec:eval_setup}

\subsection{LLM Models}

We evaluate \sys{} on four open-source instruction-tuned LLMs: Zephyr-7B~\citep{zephyr_7b}, Mistral-7B~\citep{mistral_7b}, Llama2-7B~\citep{llama2_7b}, and Llama3-8B~\citep{llama3_8b}. These models span different model families and safety-tuning recipes, allowing us to assess whether DRO reweighting provides consistent gains across models with varying inherent robustness.

For each model, we compare the undefended instruction-tuned model, the corresponding continuous adversarial training baselines, and their \sys{}-augmented variants. We denote the latter as CAT-\sys{}, CAPO-\sys{}, and MixAT-\sys{}, with suffixes $L$ and $O$ indicating learnable and optimized treatments of the dual variable $\lambda$. CAT-\sys{} and CAPO-\sys{} follow the CAT/CAPO setup of~\citet{xhonneux2024efficient_CAT}, while MixAT-\sys{} follows~\citet{dekany2025mixat}. When available, we include both released checkpoints and reproduced baselines using public code, denoted by \textsc{HF} and \textsc{R}. Each \sys{} variant is initialized from the same base model as its corresponding baseline and differs only in adversarial-loss aggregation.

\subsection{Datasets}

To isolate the effect of \sys{}, we follow the data protocol of each base method. CAT-\sys{} and CAPO-\sys{} use the CAT/CAPO adversarial prompt--response data from~\citet{xhonneux2024efficient_CAT}, while MixAT-\sys{} uses the adversarial and paraphrase-augmented data from~\citet{dekany2025mixat}. Utility datasets, when used by the base method, are kept unchanged.

For robustness evaluation, we use HarmBench~\citep{mazeika2024harmbench} and evaluate against direct harmful requests, human-written jailbreaks, AutoDAN, and GCG. Following MixAT~\citep{dekany2025mixat}, we evaluate on a fixed subset of 40 non-copyright-related HarmBench samples due to the high cost of optimization-based attacks.

For utility evaluation, we report MMLU, ARC-Easy, ARC-Challenge, and the harmless-query benchmark from~\citet{xhonneux2024efficient_CAT}, which measures benign instruction-following on 40 simple queries.

\subsection {Training Setup and Hyperparameters}

We implement \sys{} on top of the public code and configurations of the corresponding baselines. All base training components are kept unchanged, including attack generation, perturbation budget, inner attack steps, optimizer, learning-rate schedule, batch size, utility-loss weight, and training budget. The only modification is replacing the empirical average of per-sample adversarial losses with the proposed KL-DRO aggregation.

We tune only the DRO-specific parameters: the KL radius $\epsilon$, soft-constraint coefficient $\kappa$, and treatment of $\lambda$. In $\sys{}_L$, $\lambda$ is learned jointly with the model; in $\sys{}_O$, it is recomputed at each step by solving the one-dimensional convex dual problem with bisection. For each model, we select hyperparameters based on the validation robustness--utility trade-off. Full hyperparameters are listed in Appendix~\ref{app:params}.

All evaluations use the same decoding settings, attack implementations, and benchmark protocols across methods. Robustness is measured by attack success rate for each attack type and on average; utility is measured on knowledge, reasoning, and harmless-query benchmarks. Evaluation data are not used for training or hyperparameter selection.

\subsection{Hardware}
All experiments were performed on the internal and external clusters using
% {\bf \color{blue} do you mean external and internal clusters? }  
G4, 40GB A100, or 80GB A100 GPUs. All conducted experiments required at least 1,162 GPU hours.
% 1038 eval + 124 training = 

\section{Results}
\label{sec:results}

\begin{table*}[t]
\caption{Safety--utility evaluation across models and attack settings.}
\label{tab:main_results}
\centering
\small
\renewcommand{\arraystretch}{1.15}
\resizebox{\textwidth}{!}{
\begin{tabular}{llcccccccccc}
\toprule
\multirow{2}{*}{\textbf{Model}} &
\multirow{2}{*}{\textbf{Variant}} & \multicolumn{5}{c}{Attack Success Rate (ASR, \%) $\downarrow$}
& \multicolumn{5}{c}{Utility (\%) $\uparrow$} \\
\cmidrule(lr){3-7}
\cmidrule(lr){8-12}
&
& \textbf{DR}
& \textbf{HH}
& \textbf{AD}
& \textbf{GCG} 
& \textbf{Avg.}
& \textbf{MMLU} & \textbf{ARCe} & \textbf{ARCc} & \textbf{Hless} & \textbf{Avg.} \\
\midrule

%% ── Zephyr-7B ─────────────────────────────────────────────────────────────────
%% Plain row: no color scaling (neutral/white cells)
\multirow{10}{*}{\rotatebox{90}{Zephyr-7B}}
& Plain (HF)
& 82.5 & 87.5 & 95 & 77.5 & 85.63
& 59.07 & 86.74 & 73.72 & 100 & 79.88 \\ \cmidrule(r){2-12}

%% Remaining rows re-scaled among themselves (ASR: 0–22.62; Util avg: 64.94–79.8)
& CAT (HF)  
& \asrcell{2.5}{0}{22.62} 
& \asrcell{2.5}{0}{87.5} 
& \asrcell{0}{0}{95} 
& \asrcell{15}{0}{77.5}
& \asrcell{5}{0}{22.62}
& \utilitycell{56}{64.94}{79.8} 
& \utilitycell{85.56}{82.87}{87.54} 
& \utilitycell{74.31}{70.65}{74.57} 
& \utilitycell{97.5}{47.5}{100} 
& \utilitycell{78.34}{64.94}{79.8} \\ 

& CAT (R)   
& \asrcell{2.5}{0}{22.62} 
& \asrcell{2.5}{0}{87.5} 
& \asrcell{2.5}{0}{95} 
& \asrcell{22.5}{0}{77.5}
& \asrcell{7.5}{0}{22.62}
& \utilitycell{57.05}{64.94}{79.8} 
& \utilitycell{85.23}{82.87}{87.54} 
& \utilitycell{73.04}{70.65}{74.57} 
& \utilitycell{90}{47.5}{100} 
& \utilitycell{76.33}{64.94}{79.8} \\ 

& CAT-$\sys{}_L$ 
& \asrcell{2.5}{0}{22.62} 
& \asrcell{0.5}{0}{87.5} 
& \asrcell{0}{0}{95} 
& \asrcell{12.5}{0}{77.5} 
& \asrcell{3.88}{0}{22.62}
& \utilitycell{56.7}{64.94}{79.8} 
& \utilitycell{84.93}{82.87}{87.54} 
& \utilitycell{73.29}{70.65}{74.57} 
& \utilitycell{100}{47.5}{100} 
& \utilitycell{78.73}{64.94}{79.8} \\ 

& CAT-$\sys_O$ 
& \asrcell{2.5}{0}{22.62} 
& \asrcell{2}{0}{87.5} 
& \asrcell{0}{0}{95} 
& \asrcell{20}{0}{77.5}
& \asrcell{6.25}{0}{22.62}
& \utilitycell{56.54}{64.94}{79.8} 
& \utilitycell{84.85}{82.87}{87.54} 
& \utilitycell{72.54}{70.65}{74.57} 
& \utilitycell{97.5}{47.5}{100} 
& \utilitycell{77.86}{64.94}{79.8} \\ \cmidrule(r){2-12}

& CAPO (R)  
& \asrcell{25}{0}{22.62} 
& \asrcell{13}{0}{87.5} 
& \asrcell{30}{0}{95} 
& \asrcell{22.5}{0}{77.5}
& \asrcell{22.62}{0}{22.62}
& \utilitycell{58.24}{64.94}{79.8} 
& \utilitycell{86.44}{82.87}{87.54} 
& \utilitycell{74.14}{70.65}{74.57} 
& \utilitycell{97.5}{47.5}{100} 
& \utilitycell{79.8}{64.94}{79.8} \\

& CAPO-$\sys{}_L$  
& \asrcell{20}{0}{22.62} 
& \asrcell{24}{0}{87.5} 
& \asrcell{22.5}{0}{95} 
& \asrcell{12.5}{0}{77.5}
& \asrcell{19.75}{0}{22.62}
& \utilitycell{56.35}{64.94}{79.8} 
& \utilitycell{82.87}{82.87}{87.54} 
& \utilitycell{70.65}{70.65}{74.57} 
& \utilitycell{55}{47.5}{100} 
& \utilitycell{66.22}{64.94}{79.8} \\

& CAPO-$\sys_O$  
& \asrcell{17.5}{0}{22.62} 
& \asrcell{8.5}{0}{87.5} 
& \asrcell{17.5}{0}{95} 
& \asrcell{5}{0}{77.5}
& \asrcell{12.13}{0}{22.62}
& \utilitycell{56.53}{64.94}{79.8} 
& \utilitycell{83.88}{82.87}{87.54} 
& \utilitycell{71.84}{70.65}{74.57} 
& \utilitycell{47.5}{47.5}{100} 
& \utilitycell{64.94}{64.94}{79.8} \\ \cmidrule(r){2-12}

& MixAT (HF) 
& \asrcell{2.5}{0}{22.62} 
& \asrcell{2.5}{0}{87.5} 
& \asrcell{0}{0}{95} 
& \asrcell{0}{0}{77.5}
& \asrcell{1.25}{0}{22.62}
& \utilitycell{57.45}{64.94}{79.8} 
& \utilitycell{87.54}{82.87}{87.54} 
& \utilitycell{73.63}{70.65}{74.57} 
& \utilitycell{97.5}{47.5}{100} 
& \utilitycell{79.03}{64.94}{79.8} \\ 

& MixAT-$\sys_L$
& \asrcell{0}{0}{22.62} 
& \asrcell{0}{0}{87.5} 
& \asrcell{0}{0}{95} 
& \asrcell{0}{0}{77.5}
& \asrcell{0}{0}{22.62}
& \utilitycell{57.43}{64.94}{79.8} 
& \utilitycell{87.12}{82.87}{87.54} 
& \utilitycell{74.57}{70.65}{74.57} 
& \utilitycell{100}{47.5}{100} 
& \utilitycell{79.78}{64.94}{79.8} \\ \midrule

%% ── Mistral-7B ────────────────────────────────────────────────────────────────
%% Plain row: no color scaling
\multirow{7}{*}{\rotatebox{90}{Mistral-7B}}
& Plain (HF)
& 80 & 77.5 & 95 & 87.5 & 85
& 54.36 & 82.65 & 67.32 & 100 & 76.08 \\ \cmidrule(r){2-12}

%% Re-scaled among remaining rows (ASR avg: 5.75–44.12; Util avg: 73.79–75.76)
& CAT (R) 
& \asrcell{2.5}{2.5}{67.5} 
& \asrcell{2.5}{0.5}{29} 
& \asrcell{2.5}{0}{27.5} 
& \asrcell{32.5}{17.5}{55} 
& \asrcell{10}{5.75}{44.12} 
& \utilitycell{53.58}{51.46}{54.36} 
& \utilitycell{81.86}{80.89}{82.65} 
& \utilitycell{67.58}{65.27}{67.58} 
& \utilitycell{100}{95}{100} 
& \utilitycell{75.76}{73.79}{75.76} \\

& CAT-$\sys{}_L$ 
& \asrcell{2.5}{2.5}{67.5} 
& \asrcell{2.5}{0.5}{29} 
& \asrcell{0}{0}{27.5} 
& \asrcell{17.5}{17.5}{55} 
& \asrcell{5.62}{5.75}{44.12} 
& \utilitycell{53.74}{51.46}{54.36} 
& \utilitycell{80.89}{80.89}{82.65} 
& \utilitycell{66.89}{65.27}{67.58} 
& \utilitycell{95}{95}{100} 
& \utilitycell{74.13}{73.79}{75.76} \\

& CAT-$\sys_O$ 
& \asrcell{2.5}{2.5}{67.5} 
& \asrcell{0.5}{0.5}{29} 
& \asrcell{2.5}{0}{27.5} 
& \asrcell{17.5}{17.5}{55} 
& \asrcell{5.75}{5.75}{44.12} 
& \utilitycell{54.14}{51.46}{54.36} 
& \utilitycell{81.69}{80.89}{82.65} 
& \utilitycell{65.96}{65.27}{67.58} 
& \utilitycell{95}{95}{100} 
& \utilitycell{74.2}{73.79}{75.76} \\ \cmidrule(r){2-12}

& CAPO (R) 
& \asrcell{67.5}{2.5}{67.5} 
& \asrcell{29}{0.5}{29} 
& \asrcell{25}{0}{27.5} 
& \asrcell{55}{17.5}{55} 
& \asrcell{44.12}{5.75}{44.12} 
& \utilitycell{51.46}{51.46}{54.36} 
& \utilitycell{80.93}{80.89}{82.65} 
& \utilitycell{65.27}{65.27}{67.58} 
& \utilitycell{97.5}{95}{100} 
& \utilitycell{73.79}{73.79}{75.76} \\

& CAPO-$\sys{}_L$ 
& \asrcell{55}{2.5}{67.5} 
& \asrcell{19.5}{0.5}{29} 
& \asrcell{27.5}{0}{27.5} 
& \asrcell{52.5}{17.5}{55} 
& \asrcell{38.63}{5.75}{44.12} 
& \utilitycell{53.32}{51.46}{54.36} 
& \utilitycell{81.57}{80.89}{82.65} 
& \utilitycell{66.55}{65.27}{67.58} 
& \utilitycell{95}{95}{100} 
& \utilitycell{74.11}{73.79}{75.76} \\

& CAPO-$\sys_O$ 
& \asrcell{32.5}{2.5}{67.5} 
& \asrcell{7.5}{0.5}{29} 
& \asrcell{15}{0}{27.5} 
& \asrcell{27.5}{17.5}{55} 
& \asrcell{20.63}{5.75}{44.12} 
& \utilitycell{52.38}{51.46}{54.36} 
& \utilitycell{81.82}{80.89}{82.65} 
& \utilitycell{66.98}{65.27}{67.58} 
& \utilitycell{95}{95}{100} 
& \utilitycell{74.05}{73.79}{75.76} \\ \midrule

%% ── Llama2-7B ─────────────────────────────────────────────────────────────────
%% Plain row: no color scaling
\multirow{7}{*}{\rotatebox{90}{Llama2-7B}}
& Plain (HF)
& 2.5 & 5 & 5 & 32.5 & 11.25
& 45.72 & 72.26 & 55.97 & 100 & 68.49 \\ \cmidrule(r){2-12}

%% Re-scaled among remaining rows (ASR avg: 0–11.25; Util avg: 64.91–67.33)
& CAT (HF) 
& \asrcell{5}{0}{5} 
& \asrcell{5}{0}{5} 
& \asrcell{5}{0}{7.5} 
& \asrcell{30}{0}{32.5} 
& \asrcell{11.25}{0.125}{11.25} 
& \utilitycell{45.74}{44.74}{46.05} 
& \utilitycell{71.21}{71}{73.02} 
& \utilitycell{54.1}{54.01}{55.97} 
& \utilitycell{92.5}{85}{100} 
& \utilitycell{65.89}{64.91}{67.33} \\

& CAT (R)  
& \asrcell{2.5}{0}{5} 
& \asrcell{2.5}{0}{5} 
& \asrcell{5}{0}{7.5} 
& \asrcell{32.5}{0}{32.5} 
& \asrcell{10.63}{0.125}{11.25} 
& \utilitycell{45.02}{44.74}{46.05} 
& \utilitycell{71.17}{71}{73.02} 
& \utilitycell{54.1}{54.01}{55.97} 
& \utilitycell{97.5}{85}{100} 
& \utilitycell{66.95}{64.91}{67.33} \\

& CAT-$\sys{}_L$ 
& \asrcell{0}{0}{5} 
& \asrcell{1}{0}{5} 
& \asrcell{7.5}{0}{7.5} 
& \asrcell{22.5}{0}{32.5} 
& \asrcell{7.75}{0.125}{11.25} 
& \utilitycell{46.05}{44.74}{46.05} 
& \utilitycell{71}{71}{73.02} 
& \utilitycell{54.78}{54.01}{55.97} 
& \utilitycell{97.5}{85}{100} 
& \utilitycell{67.33}{64.91}{67.33} \\

& CAT-$\sys_O$
& \asrcell{2.5}{0}{5} 
& \asrcell{1}{0}{5} 
& \asrcell{2.5}{0}{7.5} 
& \asrcell{20}{0}{32.5} 
& \asrcell{6.5}{0.125}{11.25} 
& \utilitycell{44.74}{44.74}{46.05} 
& \utilitycell{71}{71}{73.02} 
& \utilitycell{54.27}{54.01}{55.97} 
& \utilitycell{97.5}{85}{100} 
& \utilitycell{66.88}{64.91}{67.33} \\ \cmidrule(r){2-12}

& CAPO (R) 
& \asrcell{0}{0}{5} 
& \asrcell{0}{0}{5} 
& \asrcell{0}{0}{7.5} 
& \asrcell{0}{0}{32.5} 
& \asrcell{0}{0.125}{11.25} 
& \utilitycell{45.84}{44.74}{46.05} 
& \utilitycell{73.02}{71}{73.02} 
& \utilitycell{55.8}{54.01}{55.97} 
& \utilitycell{85}{85}{100} 
& \utilitycell{64.91}{64.91}{67.33} \\

& CAPO-$\sys{}_L$ 
& \asrcell{0}{0}{5} 
& \asrcell{1}{0}{5} 
& \asrcell{0}{0}{7.5} 
& \asrcell{0}{0}{32.5} 
& \asrcell{0.25}{0.125}{11.25} 
& \utilitycell{45.12}{44.74}{46.05} 
& \utilitycell{71.89}{71}{73.02} 
& \utilitycell{54.27}{54.01}{55.97} 
& \utilitycell{97.5}{85}{100} 
& \utilitycell{67.2}{64.91}{67.33} \\ 

& CAPO-$\sys_O$ 
& \asrcell{0}{0}{5} 
& \asrcell{0.5}{0}{5} 
& \asrcell{0}{0}{7.5} 
& \asrcell{0}{0}{32.5} 
& \asrcell{0.125}{0.125}{11.25} 
& \utilitycell{45.22}{44.74}{46.05} 
& \utilitycell{71.97}{71}{73.02} 
& \utilitycell{54.01}{54.01}{55.97} 
& \utilitycell{97.5}{85}{100} 
& \utilitycell{67.18}{64.91}{67.33} \\ \midrule

%% ── Llama3-8B ─────────────────────────────────────────────────────────────────
%% Plain row: no color scaling
\multirow{7}{*}{\rotatebox{90}{Llama3-8B}}
& Plain (HF)
& 2.5 & 7.5 & 0 & 10 & 5
& 64.94 & 91.41 & 81.05 & 100 & 84.35 \\ \cmidrule(r){2-12}

%% Re-scaled among remaining rows (ASR avg: 0–6.25; Util avg: 74.78–83.63)
& CAT (HF) 
& \asrcell{0}{0}{2.5} 
& \asrcell{2.5}{0}{7.5} 
& \asrcell{0}{0}{0.0001} 
& \asrcell{2.5}{0}{22.5} 
& \asrcell{1.25}{0}{6.25} 
& \utilitycell{63.47}{63.46}{64.94} 
& \utilitycell{90.91}{90.31}{91.41} 
& \utilitycell{78.75}{78.75}{81.05} 
& \utilitycell{85}{65}{100} 
& \utilitycell{79.53}{74.78}{83.63} \\

& CAT (R)  
& \asrcell{2.5}{0}{2.5} 
& \asrcell{0}{0}{7.5} 
& \asrcell{0}{0}{0.0001} 
& \asrcell{22.5}{0}{22.5} 
& \asrcell{6.25}{0}{6.25} 
& \utilitycell{64.18}{63.46}{64.94} 
& \utilitycell{90.66}{90.31}{91.41} 
& \utilitycell{79.69}{78.75}{81.05} 
& \utilitycell{100}{65}{100} 
& \utilitycell{83.63}{74.78}{83.63} \\

& CAT-$\sys{}_L$ 
& \asrcell{0}{0}{2.5} 
& \asrcell{0}{0}{7.5} 
& \asrcell{0}{0}{0.0001} 
& \asrcell{0}{0}{22.5} 
& \asrcell{0}{0}{6.25} 
& \utilitycell{63.58}{63.46}{64.94} 
& \utilitycell{90.91}{90.31}{91.41} 
& \utilitycell{79.61}{78.75}{81.05} 
& \utilitycell{65}{65}{100} 
& \utilitycell{74.78}{74.78}{83.63} \\

& CAT-$\sys_O$ 
& \asrcell{0}{0}{2.5} 
& \asrcell{0.5}{0}{7.5} 
& \asrcell{0}{0}{0.0001} 
& \asrcell{17.5}{0}{22.5} 
& \asrcell{4.5}{0}{6.25} 
& \utilitycell{64}{63.46}{64.94} 
& \utilitycell{90.4}{90.31}{91.41} 
& \utilitycell{79.86}{78.75}{81.05} 
& \utilitycell{100}{65}{100} 
& \utilitycell{83.57}{74.78}{83.63} \\ \cmidrule(r){2-12}

& MixAT (HF) 
& \asrcell{0}{0}{2.5} 
& \asrcell{0}{0}{7.5} 
& \asrcell{0}{0}{0.0001} 
& \asrcell{12.5}{0}{22.5} 
& \asrcell{3.125}{0}{6.25} 
& \utilitycell{63.52}{63.46}{64.94} 
& \utilitycell{90.31}{90.31}{91.41} 
& \utilitycell{79.01}{78.75}{81.05} 
& \utilitycell{90}{65}{100} 
& \utilitycell{80.71}{74.78}{83.63} \\ 

& MixAT-$\sys{}_L$
& \asrcell{0}{0}{2.5} 
& \asrcell{0}{0}{7.5} 
& \asrcell{0}{0}{0.0001} 
& \asrcell{5}{0}{22.5} 
& \asrcell{1.25}{0}{6.25} 
& \utilitycell{63.46}{63.46}{64.94} 
& \utilitycell{90.53}{90.31}{91.41} 
& \utilitycell{78.84}{78.75}{81.05} 
& \utilitycell{100}{65}{100} 
& \utilitycell{83.21}{74.78}{83.63} \\ 

\bottomrule
\end{tabular}
}

\vspace{2pt}
\begin{minipage}{0.95\linewidth}
\footnotesize
(HF) model released on HuggingFace; (R) re-trained model or prompt using public code; Plain denotes the original basic model without any defense \\
Attack success rate (ASR) $\downarrow$ indicates lower is better for safety robustness,
while utility metrics $\uparrow$ indicate higher is better \\
\textbf{DR}: Direct Request; \textbf{HH}: Human Jailbreaks; \textbf{AD}: AutoDAN; \textbf{GCG}: Greedy Coordinate Gradient. \\
\textbf{MMLU}: Massive Multitask Language Understanding; \textbf{ARCe}: AI2 Reasoning Challenge (Easy); \textbf{ARCc}: AI2 Reasoning Challenge (Challenge); \textbf{Hless}: Harmless dataset from~\cite{xhonneux2024efficient_CAT}. \\
Cells are color-scaled within each model block and metric column, excluding Plain rows. \\
\end{minipage}

\end{table*}

We evaluate whether \sys{} improves the safety--utility trade-off of continuous adversarial training. For each base method, we compare the original adversarial training baseline with its \sys{}-augmented variant across four instruction-tuned LLMs and four attack categories. Overall, \sys{} reduces attack success rates for CAT, CAPO, and MixAT variants in most settings, while generally preserving competitive utility.

\subsection{Main Results}
\label{subsec:main_results}

Table~\ref{tab:main_results} reports robustness and utility results on Zephyr-7B-$\beta$, Mistral-7B-Instruct-v0.1, Llama-2-7B-Chat, and Llama-3-8B-Instruct. The main trend is that \sys{} improves continuous adversarial training most when the base method leaves substantial residual vulnerability. This is most evident for CAPO: CAPO-\sys{}$_O$ reduces average ASR from $22.62\%$ to $12.13\%$ on Zephyr-7B-$\beta$ and from $44.12\%$ to $20.63\%$ on Mistral-7B, while maintaining comparable average utility on Mistral-7B. These gains are consistent with the KL-DRO objective, which upweights high-loss adversarial examples and therefore targets the vulnerable tail that remains after standard adversarial training.

\sys{} also improves CAT and MixAT, with smaller absolute gains when the corresponding baselines are already strong. On Mistral-7B, CAT-\sys{}$_L$ and CAT-\sys{}$_O$ reduce average ASR from $10.00\%$ to $5.62\%$ and $5.75\%$, respectively. On Zephyr-7B-$\beta$, CAT-\sys{}$_L$ reduces average ASR to $3.88\%$, and MixAT-\sys{}$_L$ improves an already robust MixAT baseline from $1.25\%$ to $0.00\%$ ASR. On Llama-2-7B, where several baselines are already highly robust, improvements are correspondingly smaller: CAT-\sys{}$_O$ reduces average ASR from $10.63\%$ to $6.50\%$, while CAPO-\sys{} variants preserve near-zero ASR and improve average utility from $64.91\%$ to over $67\%$.

The main exception is that stronger robustness can occasionally come with over-refusal. On Llama-3-8B, CAT-\sys{}$_L$ reaches $0.00\%$ average ASR but reduces harmless-query performance to $65.00\%$. In contrast, CAT-\sys{}$_O$ preserves utility close to CAT (\textsc{R}) while reducing ASR from $6.25\%$ to $4.50\%$, and MixAT-\sys{}$_L$ gives the best trade-off in this block, reducing ASR from $3.125\%$ to $1.25\%$ while improving average utility from $80.71\%$ to $83.21\%$.

Overall, the results support our hypothesis that DRO-based adversarial-loss aggregation improves robustness by emphasizing high-loss adversarial examples. The gains are largest when the base trainer leaves a heavy vulnerable tail and smaller when the baseline is already near saturation. The improvements across CAT, CAPO, and MixAT further show that \sys{} acts as a modular loss-aggregation layer rather than a method-specific modification.

\subsection{Ablation Studies}
\label{subsec:ablation}
We conduct ablations to isolate the effect of the two main DRO-specific design choices: the KL ambiguity radius $\epsilon$ and the treatment of the dual variable $\lambda$. All ablations are performed on Mistral-7B using the same training and evaluation pipeline as in the main experiments. We report attack success rates (ASR) under direct requests, human-written jailbreaks, AutoDAN, and GCG, together with standard utility metrics.

\begin{figure}[htbp]
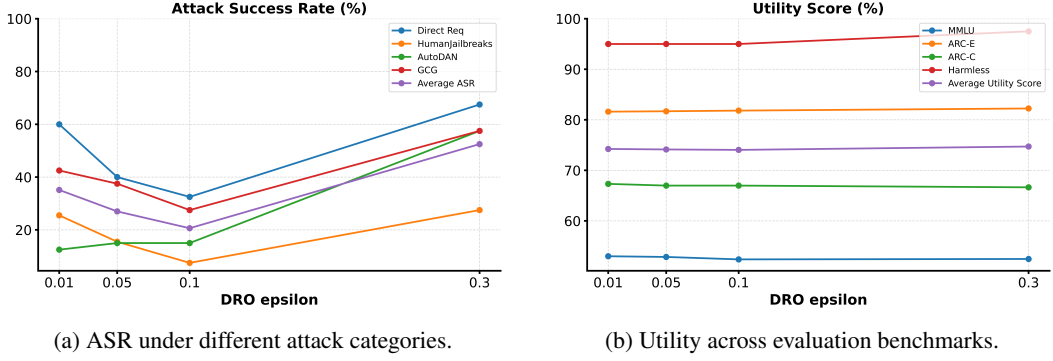

    \centering

    \begin{subfigure}[b]{0.48\textwidth}
        \centering
        \includegraphics[width=\textwidth]{Figures/dro_eps_sensitivity_asr.png}
        \caption{ASR under different attack categories.}
        \label{fig:epsilon_ablation_asr}
    \end{subfigure}
    \hfill
    \begin{subfigure}[b]{0.48\textwidth}
        \centering
        \includegraphics[width=\textwidth]{Figures/dro_eps_sensitivity_util.png}
        \caption{Utility across evaluation benchmarks.}
        \label{fig:epsilon_ablation_util}
    \end{subfigure}

    \caption{
    Sensitivity to the KL-DRO radius $\epsilon$ on Mistral-7B CAPO-\sys{}$_O$.
    Moderate values of $\epsilon$ improve robustness, with the lowest average ASR around $\epsilon=0.1$, while utility remains relatively stable across the evaluated range.
    }
    \label{fig:epsilon_ablation}
\end{figure}

\paragraph{Effect of the DRO radius $\epsilon$.}
The radius $\epsilon$ controls how far the adversarial reweighting distribution may deviate from the empirical minibatch distribution. Smaller values keep the objective close to uniform averaging, whereas larger values allow the objective to place more mass on high-loss adversarial examples. Figure~\ref{fig:epsilon_ablation} shows a non-monotonic robustness trend. Increasing $\epsilon$ initially improves robustness, with the lowest average ASR achieved around $\epsilon=0.1$. However, larger values degrade ASR, suggesting that overly aggressive reweighting can overemphasize a small set of high-loss examples and reduce training effectiveness. In contrast, utility remains relatively stable across the evaluated range. These results indicate that moderate distributional robustness is beneficial, but the ambiguity radius must be tuned to obtain the best robustness--utility trade-off.

\begin{figure}[htbp]
    \centering

    \begin{subfigure}[b]{0.31\textwidth}
        \centering
        \includegraphics[width=\textwidth]{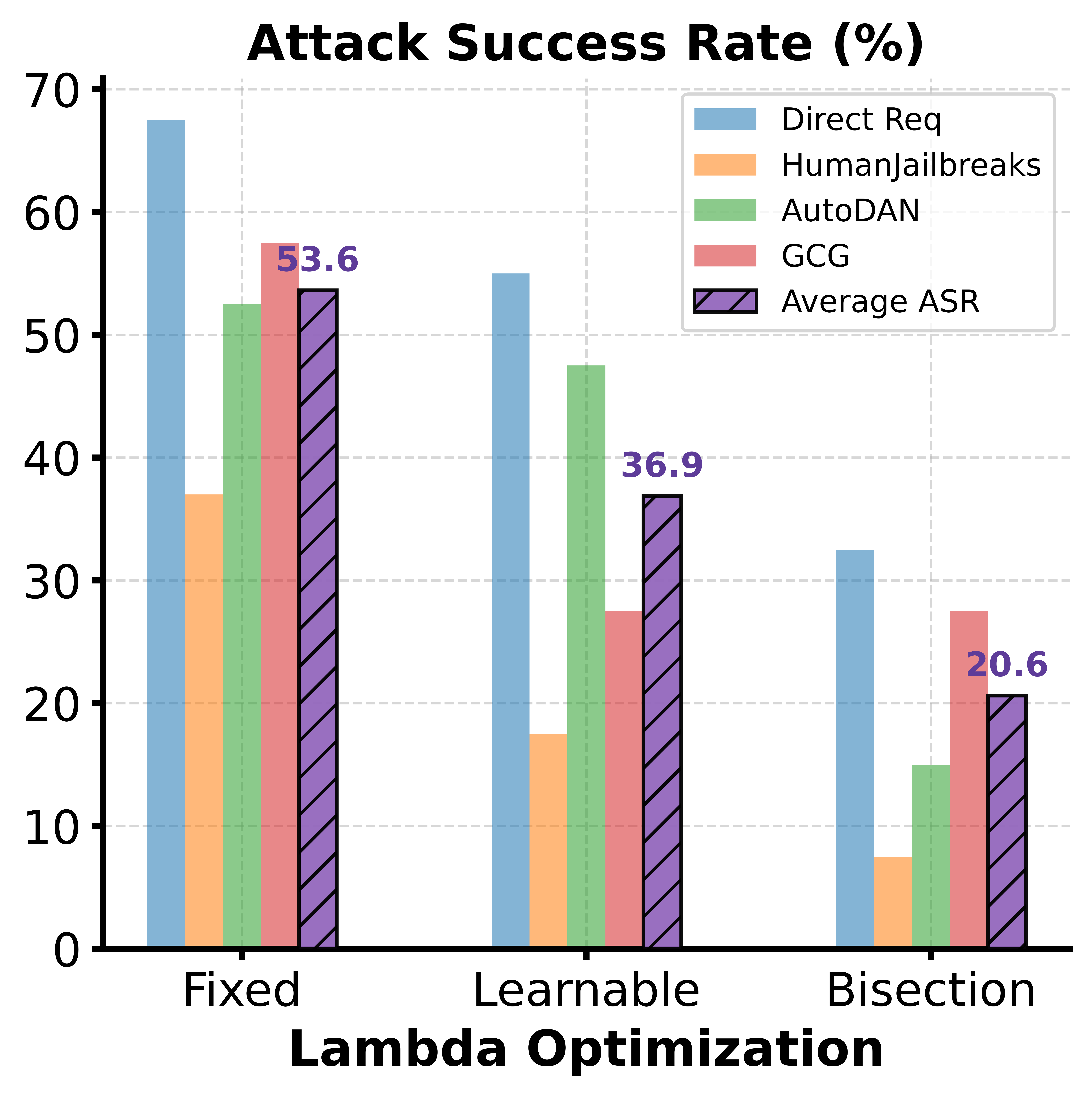}
        \caption{$\lambda=5$ / $\lambda_0=5$.}
        \label{fig:lambda_ablation_lambda5}
    \end{subfigure}
    \hfill
    \begin{subfigure}[b]{0.31\textwidth}
        \centering
        \includegraphics[width=\textwidth]{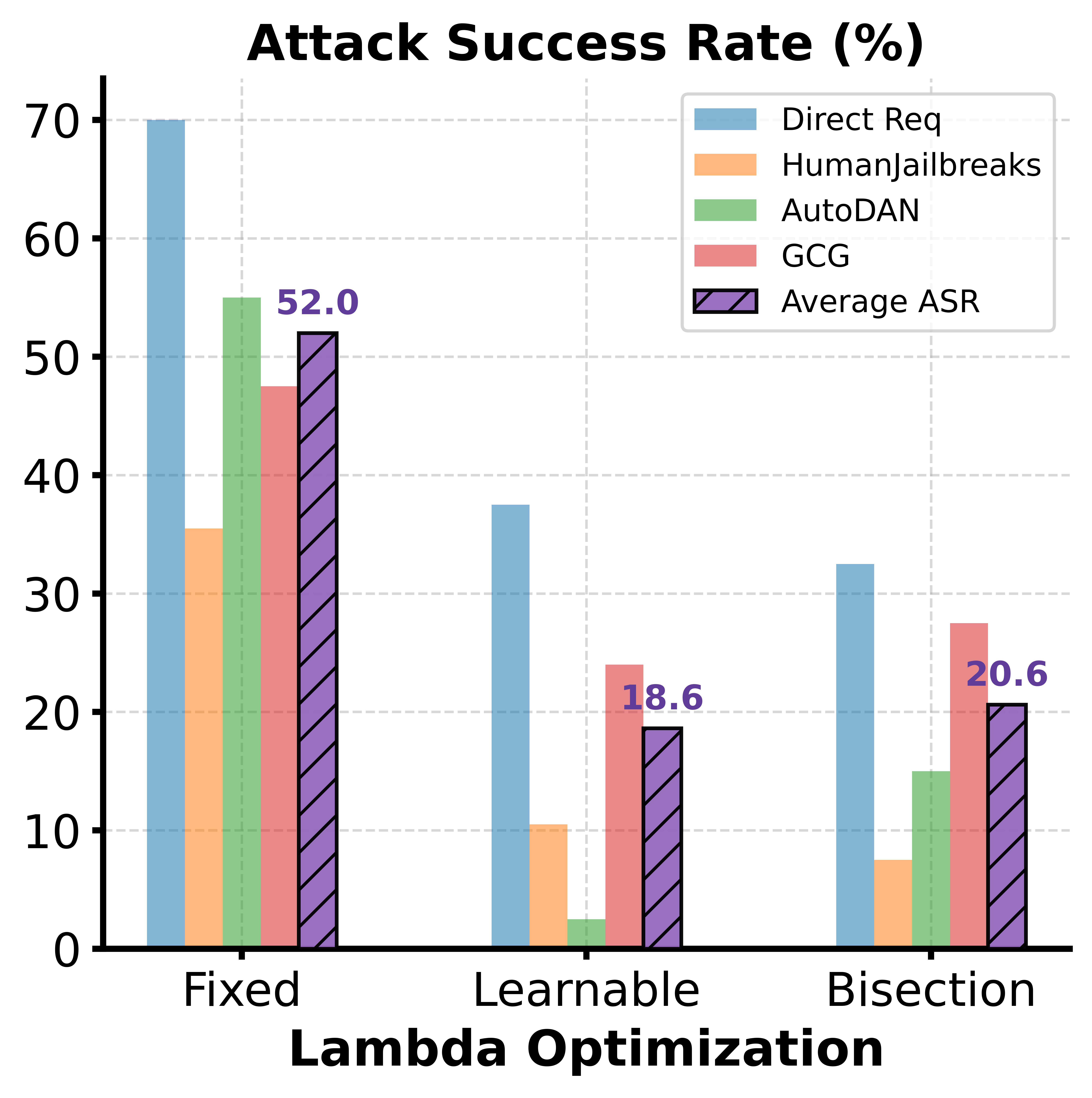}
        \caption{$\lambda=1$ / $\lambda_0=1$.}
        \label{fig:lambda_ablation_lambda1}
    \end{subfigure}
    \hfill
    \begin{subfigure}[b]{0.31\textwidth}
        \centering
        \includegraphics[width=\textwidth]{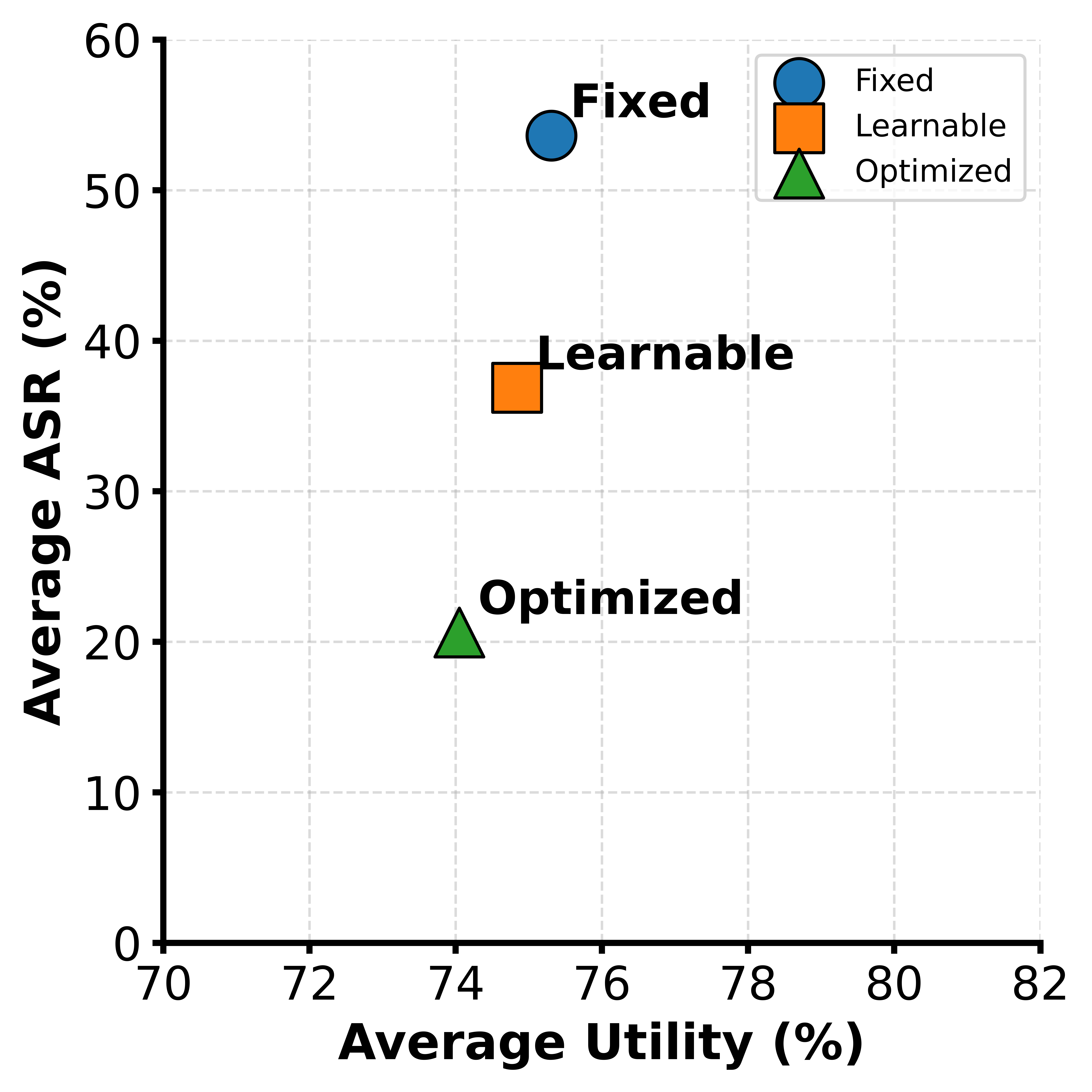}
        \caption{ASR--utility trade-off.}
        \label{fig:lambda_ablation_tradeoff}
    \end{subfigure}

    % \caption{
    % Ablation of dual-variable handling strategies on Mistral-7B.
    % For the fixed variant, $\lambda$ is held constant at the value shown in (a) or (b);
    % for $\sys_L$, this value is used as the initialization $\lambda_0$;
    % for $\sys_O$, $\lambda$ is recomputed at each training step by solving the dual problem and is therefore initialization-independent.
    % Optimizing $\lambda$ yields the lowest average ASR with comparable utility.}
    \caption{
Ablation of dual-variable handling strategies on Mistral-7B.
The fixed variant holds $\lambda$ constant; $\sys_L$ learns $\lambda$ from the shown initialization $\lambda_0$; and $\sys_O$ recomputes $\lambda$ by solving the dual problem at each step.
Optimizing $\lambda$ yields the lowest average ASR with comparable utility.
}
    \label{fig:lambda_ablation}
\end{figure}

\paragraph{Effect of the dual variable $\lambda$.}
We compare three strategies for handling the KL-DRO dual variable $\lambda$: keeping it fixed, learning it jointly with the model, and optimizing it at each training step by solving the one-dimensional convex dual problem.
For the fixed variant, $\lambda$ is held constant throughout training; for the learnable variant, the same value is used as the initialization $\lambda_0$; for the optimized variant, $\lambda$ is recomputed from the dual objective at each step and is therefore independent of the initialization.
Figures~\ref{fig:lambda_ablation_lambda5} and~\ref{fig:lambda_ablation_lambda1} show that the treatment of $\lambda$ has a substantial effect on robustness.
Fixed-$\lambda$ variants yield the highest average ASR, while learning $\lambda$ improves robustness.
The optimized variant performs best, reducing average ASR to approximately $20.6\%$, compared with $36.9\%$ for the learnable variant and $53.6\%$ for the fixed variant in the $\lambda=5$ setting.
This suggests that solving the dual problem provides a more effective minibatch-specific reweighting strength than using a predefined or slowly adapted value.

Figure~\ref{fig:lambda_training_ablation} further explains this gap through the training dynamics.
The learnable strategy changes $\lambda$ gradually from its initialization, reflecting a global gradient-based adaptation process coupled to the model update.
In contrast, the optimized strategy recomputes $\lambda^\star$ at each DRO checkpoint and rapidly moves toward smaller values.
Since smaller $\lambda$ makes the KL-DRO log-sum-exp aggregation sharper, $\sys_O$ can more quickly enter a worst-case-oriented reweighting regime when the minibatch contains hard adversarial examples.
The corresponding loss curves show that both variants reduce the training objective over time, although the batch-level loss remains noisy with occasional spikes, as expected in adversarial training.

\begin{figure}[htbp]
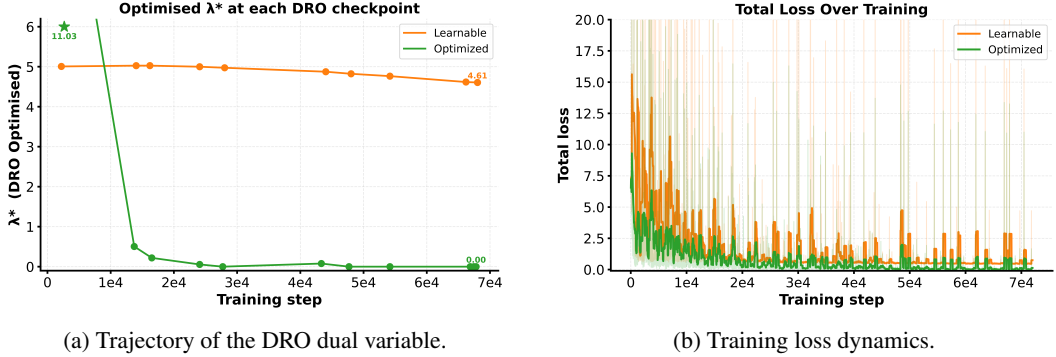

    \centering

    \begin{subfigure}[b]{0.48\textwidth}
        \centering
        \includegraphics[width=\textwidth]{Figures/lambda_training.png}
        \caption{Trajectory of the DRO dual variable.}
        \label{fig:lambda_ablation_training}
    \end{subfigure}
    \hfill
    \begin{subfigure}[b]{0.48\textwidth}
        \centering
        \includegraphics[width=\textwidth]{Figures/loss_training.png}
        \caption{Training loss dynamics.}
        \label{fig:lambda_ablation_training_loss}
    \end{subfigure}

%     \caption{Training dynamics under different DRO dual-variable treatments on Mistral-7B.
% Panel (a) plots the dual variable $\lambda$ for $\sys_L$ ($\lambda_0$ = 1) and the checkpoint-wise optimized $\lambda^\star$ for $\sys_O$.
% Panel (b) shows the corresponding total training loss.
% The optimized strategy rapidly decreases $\lambda^\star$, inducing sharper adversarial reweighting, while the learnable strategy changes more gradually.
%     }
\caption{
Training dynamics under learnable and optimized dual-variable treatments on Mistral-7B.
$\sys_O$ rapidly decreases the optimized $\lambda^\star$, inducing sharper adversarial reweighting, whereas $\sys_L$ changes $\lambda$ more gradually.
Both variants reduce the training loss, with noisy spikes typical of adversarial training.
}
    \label{fig:lambda_training_ablation}
\end{figure}

\paragraph{Robustness--utility trade-off.}
Figure~\ref{fig:lambda_ablation_tradeoff} summarizes the three $\lambda$ treatments in the ASR--utility plane.
The optimized variant achieves the lowest ASR while maintaining comparable average utility, whereas the fixed variant remains substantially less robust.
This indicates that the improvement of $\sys_O$ is not primarily due to sacrificing benign performance, but rather to selecting a sharper and more appropriate minibatch-specific reweighting strength.

Together with the $\epsilon$ ablation in Figure~\ref{fig:epsilon_ablation}, these results show that DRO reweighting must be sufficiently strong to emphasize vulnerable examples, but not so strong that it over-concentrates on a small set of high-loss samples.
This supports the central motivation of \sys{}: replacing uniform adversarial-loss aggregation with distributionally robust aggregation improves adversarial training, provided that the reweighting strength is adaptively controlled.

% The best results are therefore obtained when the DRO strength is neither fixed nor excessively aggressive, but is adapted to the current adversarial loss
% profile through intermediate $\epsilon$ values and minibatch-wise optimization of $\lambda$. These findings support the central motivation of \sys{}:
% replacing uniform adversarial-loss aggregation with distributionally robust aggregation can improve adversarial training, but the reweighting strength must be carefully controlled.

\section{Conclusion}
\label{sec:conclusion}

We presented \sys{}, a distributionally robust framework for adversarial training of LLMs. \sys{} replaces uniform aggregation of per-sample adversarial losses with an $f$-divergence DRO objective that emphasizes high-loss adversarial examples. In the KL case, convex duality yields a tractable log-sum-exp objective whose reweighting strength is controlled by a  dual variable $\lambda$. This makes \sys{} a lightweight and modular loss-aggregation layer that can be applied to existing continuous adversarial training methods without changing their attack generation pipeline.

Across multiple instruction-tuned LLMs and attack settings, \sys{} improves robustness for CAT, CAPO, and MixAT variants while generally preserving utility. The gains are largest when the base method retains nontrivial residual vulnerability, supporting the motivation that robust alignment benefits from focusing training on hard adversarial examples. Our ablations show that intermediate KL radii and adaptive optimization of $\lambda$ provide the best robustness--utility trade-off.

\paragraph{Limitations.}
\sys{} inherits the coverage and quality limitations of the adversarial data and base attack pipeline: reweighting hard observed examples does not guarantee robustness to substantially different unseen attacks. It also introduces DRO-specific hyperparameters, including $\epsilon$, $\kappa$, and the treatment of $\lambda$, which must be selected carefully to avoid weak reweighting or overemphasis on a small set of high-loss samples. Finally, our robustness evaluation follows prior work and uses a fixed HarmBench subset due to the cost of optimization-based attacks; broader evaluation on larger attack suites and larger models remains future work.

\paragraph{Broader Impact.}
This work aims to improve LLM robustness against adversarial prompting and reduce harmful outputs under jailbreak attacks. Because \sys{} is modular and lightweight, it may make robust alignment more practical across deployment settings. However, adversarial training can also increase over-refusal or degrade benign helpfulness if the safety--utility trade-off is not carefully evaluated. \sys{} should therefore be used as part of a broader safety pipeline that includes diverse red-teaming, utility and refusal-rate evaluation, monitoring for adaptive attacks, and transparent reporting of remaining failure modes.

\clearpage

\bibliography{refs}

\begin{thebibliography}{37}
\providecommand{\natexlab}[1]{#1}
\providecommand{\url}[1]{\texttt{#1}}
\expandafter\ifx\csname urlstyle\endcsname\relax
  \providecommand{\doi}[1]{doi: #1}\else
  \providecommand{\doi}{doi: \begingroup \urlstyle{rm}\Url}\fi

\bibitem[Ahmadi-Javid(2012)]{Javid2012}
A.~Ahmadi-Javid.
\newblock Entropic value-at-risk: A new coherent risk measure.
\newblock \emph{Journal of Optimization Theory and Applications}, 155:\penalty0 1105--1123, 2012.

\bibitem[Ali \& Silvey(1966)Ali and Silvey]{ali1966general}
Syed~Mumtaz Ali and Samuel~D Silvey.
\newblock A general class of coefficients of divergence of one distribution from another.
\newblock \emph{Journal of the Royal Statistical Society: Series B (Methodological)}, 28\penalty0 (1):\penalty0 131--142, 1966.

\bibitem[Altinisik et~al.(2025)Altinisik, Messaoud, Sencar, Sajjad, and Chawla]{altinisik2025explaining}
Enes Altinisik, Safa Messaoud, Husrev~Taha Sencar, Hassan Sajjad, and Sanjay Chawla.
\newblock Explaining the role of intrinsic dimensionality in adversarial training.
\newblock In \emph{International Conference on Machine Learning}, pp.\  1298--1313. PMLR, 2025.

\bibitem[Anil et~al.(2024)Anil, Durmus, Panickssery, Sharma, Benton, Kundu, Batson, Tong, Mu, Ford, et~al.]{anil2024many_redteaming}
Cem Anil, Esin Durmus, Nina Panickssery, Mrinank Sharma, Joe Benton, Sandipan Kundu, Joshua Batson, Meg Tong, Jesse Mu, Daniel Ford, et~al.
\newblock Many-shot jailbreaking.
\newblock \emph{Advances in Neural Information Processing Systems}, 37:\penalty0 129696--129742, 2024.

\bibitem[Azar et~al.(2024)Azar, Guo, Piot, Munos, Rowland, Valko, and Calandriello]{azar2024general_IPO}
Mohammad~Gheshlaghi Azar, Zhaohan~Daniel Guo, Bilal Piot, Remi Munos, Mark Rowland, Michal Valko, and Daniele Calandriello.
\newblock A general theoretical paradigm to understand learning from human preferences.
\newblock In \emph{International Conference on Artificial Intelligence and Statistics}, pp.\  4447--4455. PMLR, 2024.

\bibitem[Ben-Tal \& Teboulle(2007)Ben-Tal and Teboulle]{BenTal2007}
Aharon Ben-Tal and Marc Teboulle.
\newblock An old-new concept of convex risk measures: The optimized certainty equivalent.
\newblock \emph{Mathematical Finance}, 17\penalty0 (3):\penalty0 449--476, 2007.
\newblock \doi{10.1111/j.1467-9965.2007.00311.x}.
\newblock URL \url{https://onlinelibrary.wiley.com/doi/abs/10.1111/j.1467-9965.2007.00311.x}.

\bibitem[Birrell et~al.(2022)Birrell, Dupuis, Katsoulakis, Pantazis, and Rey-Bellet]{birrell2022f}
Jeremiah Birrell, Paul Dupuis, Markos~A Katsoulakis, Yannis Pantazis, and Luc Rey-Bellet.
\newblock $(f, {\Gamma})$-divergences: {I}nterpolating between f-divergences and integral probability metrics.
\newblock \emph{Journal of machine learning research}, 23\penalty0 (39):\penalty0 1--70, 2022.

\bibitem[Broniatowski \& Keziou(2006)Broniatowski and Keziou]{Broniatowski}
Michel Broniatowski and Amor Keziou.
\newblock Minimization of divergences on sets of signed measures.
\newblock \emph{Studia Scientiarum Mathematicarum Hungarica}, 43\penalty0 (4):\penalty0 403–442, 2006.

\bibitem[Cho et~al.(2025)Cho, Lee, and Kim]{cho2025long}
Seungju Cho, Hongsin Lee, and Changick Kim.
\newblock Long-tailed adversarial training with self-distillation.
\newblock \emph{arXiv preprint arXiv:2503.06461}, 2025.

\bibitem[Clark et~al.(2018)Clark, Cowhey, Etzioni, Khot, Sabharwal, Schoenick, and Tafjord]{clark2018think}
Peter Clark, Isaac Cowhey, Oren Etzioni, Tushar Khot, Ashish Sabharwal, Carissa Schoenick, and Oyvind Tafjord.
\newblock Think you have solved question answering? try arc, the ai2 reasoning challenge.
\newblock \emph{arXiv preprint arXiv:1803.05457}, 2018.

\bibitem[Csisz{\'a}r(1967)]{csiszar1967information}
Imre Csisz{\'a}r.
\newblock On information-type measure of difference of probability distributions and indirect observations.
\newblock \emph{Studia Sci. Math. Hungar.}, 2:\penalty0 299--318, 1967.

\bibitem[D{\'e}k{\'a}ny et~al.(2025)D{\'e}k{\'a}ny, Balauca, Staab, Dimitrov, and Vechev]{dekany2025mixat}
Csaba D{\'e}k{\'a}ny, Stefan Balauca, Robin Staab, Dimitar~I Dimitrov, and Martin Vechev.
\newblock Mix{AT}: Combining continuous and discrete adversarial training for {LLM}s.
\newblock \emph{Advances in Neural Information Processing Systems}, 2025.

\bibitem[Feng et~al.(2025)Feng, Fan, and Sun]{feng2025vulnerable}
Yuqi Feng, Jiahao Fan, and Yanan Sun.
\newblock Vulnerable data-aware adversarial training.
\newblock In \emph{The Thirty-ninth Annual Conference on Neural Information Processing Systems}, 2025.

\bibitem[Fu et~al.(2025)Fu, Ding, and Wang]{fu2025shortlength_ICDL2025_Orthogonal}
Shaopeng Fu, Liang Ding, and Di~Wang.
\newblock ''{Short-length}'' adversarial training helps {LLM}s defend ''{Long-length}'' jailbreak attacks: Theoretical and empirical evidence.
\newblock In \emph{ICLR 2025 Workshop on Foundation Models in the Wild}, 2025.
\newblock URL \url{https://openreview.net/forum?id=U74MXMriLw}.

\bibitem[Gopal et~al.(2025)Gopal, Yang, Zhang, Horton, and Chen]{gopal2025boosting}
Bhavna Gopal, Huanrui Yang, Jingyang Zhang, Mark Horton, and Yiran Chen.
\newblock Boosting adversarial robustness with clat: Criticality leveraged adversarial training.
\newblock In \emph{Forty-second International Conference on Machine Learning}, 2025.

\bibitem[Hendrycks et~al.(2021)Hendrycks, Burns, Basart, Zou, Mazeika, Song, and Steinhardt]{mmlu}
Dan Hendrycks, Collin Burns, Steven Basart, Andy Zou, Mantas Mazeika, Dawn Song, and Jacob Steinhardt.
\newblock Measuring massive multitask language understanding.
\newblock \emph{International Conference on Learning Representations}, 2021.

\bibitem[{Hugging Face H4}(2023)]{zephyr_7b}
{Hugging Face H4}.
\newblock Zephyr-7b-$\beta$.
\newblock \url{https://github.com/huggingface/alignment-handbook}, 2023.
\newblock Hugging Face model checkpoint.

\bibitem[Jiang \& Singh(2024)Jiang and Singh]{jiang2024ramp}
Enyi Jiang and Gagandeep Singh.
\newblock Ramp: Boosting adversarial robustness against multiple $ l\_p $ perturbations for universal robustness.
\newblock \emph{Advances in Neural Information Processing Systems}, 37:\penalty0 43759--43787, 2024.

\bibitem[Lee et~al.(2025)Lee, Cho, and Kim]{lee2025indirect}
Hongsin Lee, Seungju Cho, and Changick Kim.
\newblock Indirect gradient matching for adversarial robust distillation.
\newblock In \emph{13th International Conference on Learning Representations, ICLR 2025}, pp.\  49625--49646. International Conference on Learning Representations, ICLR, 2025.

\bibitem[Li \& Li(2025)Li and Li]{li2025adversarial}
Binghui Li and Yuanzhi Li.
\newblock Adversarial training can provably improve robustness: Theoretical analysis of feature learning process under structured data.
\newblock In \emph{The Thirteenth International Conference on Learning Representations}, 2025.

\bibitem[{Liese} \& {Vajda}(2006){Liese} and {Vajda}]{LieseVajda}
F.~{Liese} and I.~{Vajda}.
\newblock On divergences and informations in statistics and information theory.
\newblock \emph{IEEE Transactions on Information Theory}, 52\penalty0 (10):\penalty0 4394--4412, 2006.

\bibitem[Luenberger(1997)]{luenberger1997optimization}
D.G. Luenberger.
\newblock \emph{Optimization by Vector Space Methods}.
\newblock Professional Series. Wiley, 1997.
\newblock ISBN 9780471181170.
\newblock URL \url{https://books.google.com/books?id=M5n9DwAAQBAJ}.

\bibitem[Mazeika et~al.(2024)Mazeika, Phan, Yin, Zou, Wang, Mu, Sakhaee, Li, Basart, Li, et~al.]{mazeika2024harmbench}
Mantas Mazeika, Long Phan, Xuwang Yin, Andy Zou, Zifan Wang, Norman Mu, Elham Sakhaee, Nathaniel Li, Steven Basart, Bo~Li, et~al.
\newblock Harmbench: A standardized evaluation framework for automated red teaming and robust refusal.
\newblock \emph{arXiv preprint arXiv:2402.04249}, 2024.

\bibitem[{meta-llama}(2023)]{llama2_7b}
{meta-llama}.
\newblock Llama-2-7b-chat-hf.
\newblock \url{https://huggingface.co/meta-llama/Llama-2-7b-chat-hf}, 2023.
\newblock Hugging Face model checkpoint.

\bibitem[{meta-llama}(2024)]{llama3_8b}
{meta-llama}.
\newblock Meta-llama-3-8b-instruct.
\newblock \url{https://huggingface.co/meta-llama/Meta-Llama-3-8B-Instruct}, 2024.
\newblock Hugging Face model checkpoint.

\bibitem[{mistralai}(2025)]{mistral_7b}
{mistralai}.
\newblock Mistral-7b-instruct-v0.1.
\newblock \url{https://huggingface.co/mistralai/Mistral-7B-v0.1}, 2025.
\newblock Hugging Face model checkpoint.

\bibitem[{Nguyen} et~al.(2010){Nguyen}, {Wainwright}, and {Jordan}]{Nguyen_Full_2010}
X.~{Nguyen}, M.~J. {Wainwright}, and M.~I. {Jordan}.
\newblock Estimating divergence functionals and the likelihood ratio by convex risk minimization.
\newblock \emph{IEEE Transactions on Information Theory}, 56\penalty0 (11):\penalty0 5847--5861, 2010.

\bibitem[Ponstein(2004)]{ponstein2004approaches}
J.~Ponstein.
\newblock \emph{Approaches to the Theory of Optimization}.
\newblock Cambridge Tracts in Mathematics. Cambridge University Press, 2004.
\newblock ISBN 9780521604918.
\newblock URL \url{https://books.google.com/books?id=GaNB2B677wgC}.

\bibitem[Shaopeng et~al.(2026)Shaopeng, Fu, and Wang]{CAT_ICLR_26}
Shaopeng, Di~Fu, and Wang.
\newblock Understanding and improving continuous {LLM} adversarial training via in-context learning theory.
\newblock \emph{Internatial Conference on Learning Representations}, 2026.

\bibitem[Sheshadri et~al.(2025)Sheshadri, Ewart, Guo, Lynch, Wu, Hebbar, Sleight, Stickland, Perez, Hadfield-Menell, and Casper]{sheshadri2025latent_latent}
Abhay Sheshadri, Aidan Ewart, Phillip~Huang Guo, Aengus Lynch, Cindy Wu, Vivek Hebbar, Henry Sleight, Asa~Cooper Stickland, Ethan Perez, Dylan Hadfield-Menell, and Stephen Casper.
\newblock Latent adversarial training improves robustness to persistent harmful behaviors in {LLM}s.
\newblock \emph{Transactions on Machine Learning Research}, 2025.
\newblock ISSN 2835-8856.
\newblock URL \url{https://openreview.net/forum?id=6LxMeRlkWl}.

\bibitem[Wang et~al.(2025)Wang, Wu, Ji, Li, Ma, Wang, Li, Liu, Liu, and Rahmel]{wang2025selfdefend_Usenix_Orthogonal}
Xunguang Wang, Daoyuan Wu, Zhenlan Ji, Zongjie Li, Pingchuan Ma, Shuai Wang, Yingjiu Li, Yang Liu, Ning Liu, and Juergen Rahmel.
\newblock {SelfDefend}: {LLMs} can defend themselves against jailbreaking in a practical manner.
\newblock In \emph{34th USENIX Security Symposium (USENIX Security 25)}, pp.\  2441--2460, 2025.

\bibitem[Wang et~al.(2024)Wang, Zhang, and Arora]{wang2024benign}
Yunjuan Wang, Kaibo Zhang, and Raman Arora.
\newblock Benign overfitting in adversarial training of neural networks.
\newblock In \emph{Forty-first International Conference on Machine Learning}, 2024.

\bibitem[Xhonneux et~al.(2024)Xhonneux, Sordoni, G{\"u}nnemann, Gidel, and Schwinn]{xhonneux2024efficient_CAT}
Sophie Xhonneux, Alessandro Sordoni, Stephan G{\"u}nnemann, Gauthier Gidel, and Leo Schwinn.
\newblock Efficient adversarial training in {LLM}s with continuous attacks.
\newblock \emph{Advances in Neural Information Processing Systems}, 37:\penalty0 1502--1530, 2024.

\bibitem[Xiao et~al.(2024)Xiao, Zhang, Luo, and Ozdaglar]{xiao2024uniformly}
Jiancong Xiao, Jiawei Zhang, Zhi-Quan Luo, and Asuman Ozdaglar.
\newblock Uniformly stable algorithms for adversarial training and beyond.
\newblock \emph{arXiv preprint arXiv:2405.01817}, 2024.

\bibitem[Xie \& Huo(2024)Xie and Huo]{xie2024high}
Yiling Xie and Xiaoming Huo.
\newblock High-dimensional (group) adversarial training in linear regression.
\newblock \emph{Advances in Neural Information Processing Systems}, 37:\penalty0 31708--31735, 2024.

\bibitem[Zhang et~al.(2024{\natexlab{a}})Zhang, Liu, Zhou, Zhang, and Liu]{zhang2024improving}
Jiacheng Zhang, Feng Liu, Dawei Zhou, Jingfeng Zhang, and Tongliang Liu.
\newblock Improving accuracy-robustness trade-off via pixel reweighted adversarial training.
\newblock In \emph{Proceedings of the 41st International Conference on Machine Learning}, pp.\  59382--59402, 2024{\natexlab{a}}.

\bibitem[Zhang et~al.(2024{\natexlab{b}})Zhang, Wang, and Arora]{zhang2024stability}
Kaibo Zhang, Yunjuan Wang, and Raman Arora.
\newblock Stability and generalization of adversarial training for shallow neural networks with smooth activation.
\newblock \emph{Advances in Neural Information Processing Systems}, 37:\penalty0 16160--16193, 2024{\natexlab{b}}.

\end{thebibliography}
\bibliographystyle{iclr2024_conference}

\appendix

\section{DRO Dual Formulation}\label{app:DRO_duality_proof}

In this appendix we provide a derivation of the DRO reformulation result via convex duality; the argument is similar to the proof of Theorem 5.1 in \cite{Javid2012}, with the main difference being the inclusion of the $f$-divergence soft-constraint penalty term $\kappa D_f(Q\|P_n)$, which we require for numerical stability purposes; background on $f$-divergences can be found in Appendix \ref{app:f_div}.
\begin{theorem}\label{thm:duality}
Let $\mathcal{Z}$ be a measurable space  and $P_n$ be the empirical distribution of a collection of samples $z_i \in \mathcal{Z}$, $i=1,...,n$.  Further, suppose:
\begin{enumerate}
    \item We have $0\leq a<1<b\leq \infty$ and a convex $f:(a,b)\to\mathbb{R}$ that satisfies $f(1)=0$.
    \item  We have a measurable loss $\mathcal{L}:\mathcal{Z}\to\mathbb{R}$.
\end{enumerate}
Then for all $\epsilon>0$, $\kappa> 0$ we have the equality 
\begin{align}\label{eq:reformulation_final_app}
&\sup_{Q: D_f(Q\|P_n)\leq \epsilon}\{E_Q[\mathcal{L}]-\kappa D_f(Q\|P_n)\}\\
=&\inf_{\substack{\lambda\geq0,\\\rho\in\mathbb{R}} }\!\left\{\!\lambda \epsilon+\rho+(\lambda+\kappa)E_{P_n}\!\!\left[f^*\!\!\left(\frac{\mathcal{L}-\rho}{\lambda+\kappa}\right)\!\right]\right\}\notag\,,
\end{align}
where $f^*$ denotes the Legendre transform of $f$. In the KL case we have the further simplification
  \begin{align}\label{eq:reformulation_final_KL_app}
&\sup_{Q:\mathrm{KL}(Q\|P_n)\leq \epsilon} \{E_{Q}[\mathcal{L}]+\kappa \mathrm{KL}(Q\|P_n)\} \\
=&\inf_{\lambda\geq0}\!\left\{\lambda\epsilon+(\lambda+\kappa)\log E_{P_n}[\exp(\mathcal{L}/(\lambda+\kappa))]\right\}\,.\notag
    \end{align}
  Moreover, the objective function on the right-hand side of \eqref{eq:reformulation_final_app} is convex in $(\lambda,\rho)$ and the objective  on the right-hand side of \eqref{eq:reformulation_final_KL_app} is convex in $\lambda$.
\end{theorem}
\begin{proof}
 We will employ  the Slater condition for convex duality, see, e.g., Theorem 3.11.2  in \cite{ponstein2004approaches} and Theorem 8.3.1 and Problem 8.7 in \cite{luenberger1997optimization}.  To that end, let $V$ be the vector space of finite signed measures on $\mathcal{Z}$ and define $W$ to be subset of probability measures satisfying $D_f(Q\|P_n)<\infty$. Next recall  that the  $f$-divergences are convex, which is a consequence of their  variational representation   \citep{Broniatowski,Nguyen_Full_2010,birrell2022f}.  This implies that $W$ is a  convex subset of $V$.  Define the convex functional $F:W\to \mathbb{R}$ by $F[Q]=E_Q[-\mathcal{L}]+\kappa D_f(Q\|P_n)$; we note that $Q\in W$ implies   $Q\ll P_n$, and so  $Q$ is supported on $\{z_i\}_{i=1}^n$, which implies $F[Q]$ is finite as claimed.  Convexity of $F$ follows from convexity of the $f$-divergence and linearity of the integral.  Finally, note that the convex inequality constraint $D_f(Q\|P_n)\leq \epsilon$ is strictly satisfied at $Q=P_n\in W$, i.e., $D_f(P_n\|P_n)=0<\epsilon$.   Thus we have confirmed the Salter condition and can conclude the strong duality result
\begin{align}\label{eq:strong_duality_app}
 &\inf_{Q:D_f(Q\|P_n)\leq \epsilon}\left\{E_Q[-\mathcal{L}]+\kappa D_f(Q\|P_n)\right\}\\
 =&\sup_{\lambda\geq 0 }\left\{-\lambda \epsilon+\inf_{Q:D_f(Q\|P_n)<\infty}\{E_Q[-\mathcal{L}]+(\lambda+\kappa) D_f(Q\|P_n)\}\right\}\,.\notag
 \end{align}
 Finally, for all $\lambda\geq 0$ the inner minimization over $Q$ can be evaluated by using the Gibbs variational formula for $f$-divergences, see Theorem 4.2 in \cite{BenTal2007} (note that the requirement that the minimum of $f$ equal zero is an inessential normalization condition; see, e.g, \cite{birrell2022f}). Specifically, after multiplying both sides by $-1$ this yields
\begin{align}\label{eq:Gibbs_app}
 &\sup_{Q:D_f(Q\|P_n)\leq \epsilon}\left\{E_Q[\mathcal{L}]-\kappa D_f(Q\|P_n)\right\}\\
 =&\inf_{\lambda\geq 0 }\left\{\lambda \epsilon+(\lambda+\kappa)\sup_{Q:D_f(Q\|P_n)<\infty}\{E_Q[\mathcal{L}/(\lambda+\kappa)]- D_f(Q\|P_n)\}\right\}\notag\\
=&\inf_{\lambda\geq 0 }\left\{\lambda \epsilon+(\lambda+\kappa)\inf_{\nu\in\mathbb{R}}\left\{\nu+E_{P_n}[f^*(\mathcal{L}/(\lambda+\kappa)-\nu)]\right\} \right\}\notag\,.
 \end{align}
Changing variables in the inner minimization from $\nu$ to $\rho=(\lambda+\kappa)\nu$, we arrive at \eqref{eq:reformulation_final_app}.
The KL result \eqref{eq:reformulation_final_KL_app} then follows from the fact that $f^*_{\mathrm{KL}}(y)=e^{y-1}$, which allows the optimization over  $\rho$  to be computed explicitly  via a straightforward calculus argument.

Convexity  in $(\lambda,\rho)$ of the objective on the right-hand side of \eqref{eq:reformulation_final_app}   follows from applying standard convex analysis results, including the convexity of the Legendre transform, $f^*$, and  convexity of the perspective of a convex function. Finally, convexity of the objective on the right-hand side of \eqref{eq:reformulation_final_KL_app} follows from the fact that minimizing a jointly convex function over one of its arguments results in  a convex  function in the remaining argument.
\end{proof}

\subsection{Background on $f$-Divergences}\label{app:f_div}
The $f$-divergences,   introduced by  \cite{ali1966general,csiszar1967information}, quantify the discrepancy between a pair of probability distributions, $Q$ and $P$, by using the likelihood ratio $dQ/dP$, i.e., by comparing the relative weights under $Q$ and $P$. More precisely, for any convex function $f$ on the real line with $f(1)=0$, the corresponding $f$-divergence is defined by
\begin{align}\label{eq:f_div_def}
    D_f(Q\|P)\coloneqq \begin{cases}
E_P[f(dQ/dP)] &\text{ if }Q\ll P\\
\infty &\text{ otherwise}
\end{cases}\,,
\end{align}
where $Q\ll P$ denotes absolute continuity, i.e., existence of the likelihood ratio. The most widely used   $f$-divergence is the KL divergence (i.e., relative entropy), which is defined in terms of $f_{\mathrm{KL}}(t)\coloneqq t\log(t)$.   For further discussions of the properties of  $f$-divergences  see, e.g., \cite{LieseVajda,birrell2022f}.

\section{Algorithm Pseudocode}
\label{app:pseudo}

\begin{algorithm}[H]
\caption{Distributionally Robust LLM Adversarial Training}
\label{alg:dro-llm-at}
\begin{algorithmic}[1]
\REQUIRE 
Adversarial data $\{(x_i,y_i,\hat y_i)\}_{i=1}^{n}$,
utility data $\mathcal{D}_u$,
target LLM $f_\theta$,
base adversarial method $\mathsf{BaseAT}$,
number of training iterations $N$,
minibatch size $B$,
number of inner attack iterations $M$,
\sys radius $\epsilon$,
soft regularization parameter $\kappa$,
dual variable initialization $\lambda_0$,
\sys mode $m_\lambda \in \{\mathrm{learnable},\mathrm{optimized}\}$,
learning rates $\eta_{\delta},\eta_{\lambda},\eta_{\theta}$.
\ENSURE Robustified LLM $f_\theta$.

\STATE Initialize $\lambda\leftarrow \lambda_0$.

\FOR{$n=1,\ldots,N$}
    \STATE Sample adversarial minibatch $\mathcal{B}_{n}$.
    \STATE Sample utility minibatch $\mathcal{U}_{n}$, if used.

    \FOR{$(x_i,y_i,\hat y_i)\in \mathcal{B}_{n}$}
        \STATE Initialize adversarial perturbation $\delta_i$.
        \FOR{$m=1,\ldots,M$}
            \STATE Update perturbation using the base attack:
            \[
            \delta_i
            \leftarrow
            \Pi_{\Delta}
            \left(
            \delta_i
            +
            \eta_{\delta}
            \nabla_{\delta}
            A_{attack}
            (x_i,y_i,\hat y_i,\delta_i,\theta)
            \right).
            \]
        \ENDFOR
        \STATE Set $\tilde x_i \leftarrow x_i+\delta_i$.
        \STATE Compute per-sample adversarial loss:
        \[
        L_i^{adv}
        \leftarrow
        \mathsf{Loss}_{\mathsf{BaseAT}}
        (f_\theta,\tilde x_i,y_i,\hat y_i).
        \]
    \ENDFOR

    \STATE Let $\mathbf{L}^{adv}=(L_1^{adv},\ldots,L_B^{adv})$.

    \STATE \textbf{Auxiliary update for the \sys dual variable $\lambda$:}
    \[
    \lambda
    \leftarrow
    \begin{cases}
    \Pi_{[\lambda_{\min},\lambda_{\max}]}
    \left(
    \lambda-\eta_{\lambda}
    \nabla_{\lambda}
    A_{DRO}^{KL}(\mathbf{L}^{adv},\lambda)
    \right),
    & m_\lambda=\mathrm{learnable}, \\[1.2em]
    \displaystyle
    \arg\min_{\lambda\ge 0}
    A_{DRO}^{KL}(\mathbf{L}^{adv},\lambda),
    & m_\lambda=\mathrm{optimized}.
    \end{cases}
    \]

    \STATE Compute the KL-DRO adversarial objective:
    \[
    A_{DRO}^{KL}
    (\mathbf{L}^{adv},\lambda)
    =
    \lambda\epsilon
    +
    (\lambda+\kappa)
    \log
    \left(
    \frac{1}{B}
    \sum_{i=1}^{B}
    \exp
    \left(
    \frac{L_i^{adv}}{\lambda+\kappa}
    \right)
    \right).
    \]

    \IF{utility data is used}
        \STATE Compute standard utility loss:
        \[
        \mathcal{L}_{u}
        =
        -\frac{1}{|\mathcal{U}_{n}|}
        \sum_{(x,y)\in\mathcal{U}_{n}}
        \log f_\theta(y|x).
        \]
        \STATE Set total loss:
        \[
        \mathcal{L}
        =
        A_{DRO}^{KL}
        (\mathbf{L}^{adv},\lambda)
        +
        \alpha_u\mathcal{L}_{u}.
        \]
    \ELSE
        \STATE Set total loss:
        \[
        \mathcal{L}
        =
        A_{DRO}^{KL}
        (\mathbf{L}^{adv},\lambda).
        \]
    \ENDIF

    \STATE Update model parameters:
    \[
    \theta
    \leftarrow
    \theta
    -
    \eta_\theta
    \nabla_\theta \mathcal{L}.
    \]
\ENDFOR

\RETURN Robustified LLM $f_\theta$.
\end{algorithmic}
\end{algorithm}

\section{Optimization over $\lambda$}\label{app:bisection_lambda}
In this appendix we detail the implementation of the $m_\lambda=\text{optimized}$ case of Algorithm \ref{alg:dro-llm-at}, line 14.  Here, the goal is to solve
\begin{align}\label{eq:lambda_opt_app}
\arg\min_{\lambda\ge 0}
    A_{DRO}^{KL}(\mathbf{L}^{adv},\lambda)\,,
\end{align}
where the objective is given by
\begin{align}
A_{DRO}^{KL}
    (\mathbf{L}^{adv},\lambda)
    =
    \lambda\epsilon
    +
    (\lambda+\kappa)
    \log
    \left(
    \frac{1}{B}
    \sum_{i=1}^{B}
    \exp
    \left(
    \frac{L_i^{adv}}{\lambda+\kappa}
    \right)
    \right)\,.
\end{align}
Recall that $\epsilon,\kappa>0$ and $L_i^{adv}$ are the values of the loss on the adversarial samples.

We start by noting that $A_{DRO}^{KL}
    (\mathbf{L}^{adv},\lambda)$ is smooth and convex in $\lambda\in[0,\infty)$ (see Theorem \ref{thm:duality}).  Now consider the following two cases.

{\bf Case 1:}   $\partial_\lambda|_{\lambda=0} A_{DRO}^{KL}
    (\mathbf{L}^{adv},\lambda)< 0$\\
First note that $\lim_{\lambda\to\infty} A_{DRO}^{KL}
    (\mathbf{L}^{adv},\lambda)=\infty$. This implies the existence of  $\lambda_r>0$ such that $\partial_\lambda|_{\lambda=\lambda_r} A_{DRO}^{KL}
    (\mathbf{L}^{adv},\lambda)>0$.    Combined with the condition that defines this case, we can therefore conclude existence of $\lambda_*\in[0,\lambda_r]$ such that  $\partial_\lambda|_{\lambda=\lambda_*} A_{DRO}^{KL}
    (\mathbf{L}^{adv},\lambda)=0$. Hence, by standard convex analysis theory, the critical point $\lambda_*$ is  a global minimizer.

{\bf Case 2:}   $\partial_\lambda|_{\lambda=0} A_{DRO}^{KL}
    (\mathbf{L}^{adv},\lambda)\geq 0$\\
Convexity implies that the derivative is non-decreasing, hence in this case we can conclude that $A_{DRO}^{KL}
    (\mathbf{L}^{adv},\lambda)$ is non-decreasing in $\lambda\in[0,\infty)$. Therefore the minimum is achieved at $\lambda_*=0$.

Based on these two cases, our approach for solving \eqref{eq:lambda_opt_app} is to first check the sign of $\partial_\lambda|_{\lambda=0} A_{DRO}^{KL}
    (\mathbf{L}^{adv},\lambda)$.  If it is non-negative then, according to Case 2, we return $\lambda_*=0$.  If it is negative then we use the bisection method to search for a critical point which, per Case 1, will be an optimizer.  Specifically, we initialize a variable $\lambda_r>0$ and increase it until we observe a sign change in the derivative, i.e., until $\partial_\lambda|_{\lambda=\lambda_r} A_{DRO}^{KL}
    (\mathbf{L}^{adv},\lambda)>0$ (as noted in Case 1, such a sign change is guaranteed to occur).  Subsequently we apply a standard bisection method solver (e.g., \texttt{optimize.bisect} in SciPy)
to return a solution, $\lambda_*$, to $\partial_\lambda  A_{DRO}^{KL}
    (\mathbf{L}^{adv},\lambda)=0$ in the interval $[0,\lambda_r]$.  Finally, we note the following
explicit  formula for the required derivative:
\begin{align}
   \partial_{\lambda}    A_{DRO}^{KL}(\mathbf{L}^{adv},\lambda)= \epsilon
    +
    \log
    \left(
    \frac{1}{B}
    \sum_{i=1}^{B}
    \exp
    \left(
    \frac{L_i^{adv}}{\lambda+\kappa}
    \right)
    \right)- (\lambda+\kappa)^{-1}
\frac{\frac{1}{B}
    \sum_{i=1}^{B}
  L_i^{adv}  \exp
    \left(
    \frac{L_i^{adv}}{\lambda+\kappa}
    \right) }{
    \frac{1}{B}
    \sum_{i=1}^{B}
    \exp
    \left(
    \frac{L_i^{adv}}{\lambda+\kappa}
    \right)
}\,.   
\end{align}
This formula illustrates  the importance of using a positive  $\kappa$, as it prevents a singularity at $\lambda=0$.

\section{Datasets, Models, and Licenses}
\label{app:assets}

We use publicly available datasets, model checkpoints, and baseline implementations. We credit the original creators through citations in the main text and summarize the sources and licenses of the main assets below. We do not redistribute third-party model weights, datasets, or baseline code as part of this submission.

\paragraph{Training and evaluation data.}
For adversarial training, CAT-\sys{} and CAPO-\sys{} use the adversarial prompt--response data from~\citet{xhonneux2024efficient_CAT}, while MixAT-\sys{} follows the adversarial and paraphrase-augmented data protocol of~\citet{dekany2025mixat}. When a utility dataset is used by the corresponding base method, we keep it unchanged. For robustness evaluation, we use HarmBench~\citep{mazeika2024harmbench} and evaluate on direct harmful requests, human-written jailbreaks, AutoDAN, and GCG attacks. For utility evaluation, we use MMLU~\citep{mmlu}, ARC-Easy and ARC-Challenge~\citep{clark2018think}, and the harmless-query benchmark from~\citet{xhonneux2024efficient_CAT}.

\paragraph{Models.}
We evaluate \sys{} on publicly available instruction-tuned LLMs. Table~\ref{tab:model_licenses} lists the model checkpoints used in this work. We use these models for research evaluation only and do not redistribute their weights.

\begin{table}[h]
\centering
\caption{Model checkpoints used in this work. Licenses and access terms should be verified against the corresponding official model cards before final submission.}
\label{tab:model_licenses}
\begin{tabular}{lll}
\toprule
Model & Source & License / terms \\
\midrule
Zephyr-7B-$\beta$~\citep{zephyr_7b} & Hugging Face H4 & MIT License \\
Mistral-7B-Instruct-v0.1~\citep{mistral_7b} & Mistral AI & Apache-2.0 License \\
Llama-2-7B-Chat~\citep{llama2_7b} & Meta & Llama 2 Community License \\
Llama-3-8B-Instruct~\citep{llama3_8b} & Meta & Meta Llama 3 Community License \\
\bottomrule
\end{tabular}
\end{table}

\paragraph{Baseline implementations.}
CAT-\sys{} and CAPO-\sys{} follow the CAT/CAPO implementation and training setup of~\citet{xhonneux2024efficient_CAT}, while MixAT-\sys{} follows the MixAT implementation and setup of~\citet{dekany2025mixat}. Where public code or checkpoints are available, we use them only for research reproduction and comparison, and cite the original papers and repositories. We do not redistribute third-party baseline code or checkpoints.

\section{Adversarial Training Hyperparameters}
\label{app:params}

\begin{table}[htbp]
\centering
\caption{
DRO-specific hyperparameters for \sys{} variants.
Subscripts $L$ and $O$ denote learnable and optimized treatments of the dual variable $\lambda$, respectively.
All non-DRO hyperparameters are inherited from the corresponding base method.
}
\label{tab:dro_hyperparameters}
\resizebox{\textwidth}{!}{
\begin{tabular}{llccccc}
\toprule
\textbf{Model} 
& \textbf{Variant} 
& \textbf{$\lambda$ treatment} 
& \textbf{$\lambda_0$ / solver setting} 
& \textbf{$\lambda$-LR}
& \textbf{KL radius $\epsilon$} 
& \textbf{Soft penalty $\kappa$} \\
\midrule
\multirow{5}{*}{Zephyr-7B} & CAT-\sys{}$_L$   & learnable & 5 & 0.0001 & 0.05 & 0.08 \\
 & CAT-\sys{}$_O$   & optimized & bisection & -- & 0.1 & 0.08\\
 & CAPO-\sys{}$_L$  & learnable & 5 & 0.0001 & 0.08 & 0.053 \\
 & CAPO-\sys{}$_O$  & optimized & bisection & -- & 0.1 & 0.07 \\
 & MixAT-\sys{}$_L$ & learnable & 5 & 0.001 & 0.1 & 0.1 \\
\midrule
\multirow{4}{*}{Mistral-7B} & CAT-\sys{}$_L$   & learnable & 5 & 0.0001 & 0.3 & 0.1 \\
& CAT-\sys{}$_O$   & optimized & bisection & -- & 0.05 & 0.1 \\
& CAPO-\sys{}$_L$  & learnable & 5 & 0.001 & 0.1 & 0.3 \\
& CAPO-\sys{}$_O$  & optimized & bisection & -- & 0.1 & 0.1 \\
\midrule
\multirow{4}{*}{Llama2-7B} & CAT-\sys{}$_L$   & learnable & 5 & 0.00001 & 0.8 & 0.5 \\
& CAT-\sys{}$_O$   & optimized & bisection & -- & 0.5 & 0.1 \\
& CAPO-\sys{}$_L$  & learnable & 5 & 0.001 & 0.2 & 0.1 \\
& CAPO-\sys{}$_O$  & optimized & bisection & -- & 0.1 & 0.1 \\
\midrule
\multirow{3}{*}{Llama3-8B} & CAT-\sys{}$_L$   & learnable & 5 & 0.00001 & 0.55 & 0.5 \\
& CAT-\sys{}$_O$   & optimized & bisection & -- & 0.5 & 0.1 \\
& MixAT-\sys{}$_L$ & learnable & 5 & 0.001 & 0.1 & 0.1 \\
\bottomrule
\end{tabular}
}
\end{table}

We tune only the hyperparameters introduced by \sys{}: the KL ambiguity radius $\epsilon$, the soft regularization coefficient $\kappa$, and the treatment of the dual variable $\lambda$. All other training hyperparameters, including the optimizer, learning-rate schedule, batch size, perturbation budget, number of inner attack steps, utility-loss weight, sequence length, and total training budget, are inherited from the corresponding base adversarial training method. Specifically, CAT-\sys{} and CAPO-\sys{} follow the CAT/CAPO configurations of~\cite{xhonneux2024efficient_CAT}, while MixAT-\sys{} follows the MixAT configuration of~\cite{dekany2025mixat}.

For each model and base method, we select DRO-specific hyperparameters using validation robustness--utility trade-off. Table~\ref{tab:dro_hyperparameters} reports the selected configurations used in the main experiments. In \sys{}$_L$, $\lambda$ is initialized at $\lambda_0$ and optimized jointly with the model using a separate dual learning rate.
In \sys{}$_O$, $\lambda$ is recomputed at each training step by solving the one-dimensional convex KL-DRO dual problem with a bisection solver.

% % \newpage
% \clearpage
% \input{checklist.tex}

\end{document}